\documentclass{article}

\usepackage{microtype}
\usepackage{graphicx}
\usepackage{subfigure}
\usepackage{booktabs}
\usepackage{pinlabel}
\usepackage{sidecap}
\sidecaptionvpos{figure}{t}

\usepackage{hyperref}

\usepackage[accepted]{icml2023}

\usepackage{amsmath}
\usepackage{amssymb}
\usepackage{mathtools}
\usepackage{amsthm}

\usepackage{enumitem}
\usepackage{amscd}
\usepackage{epsf}
\usepackage{latexsym}
\usepackage{verbatim}
\usepackage{color}
\usepackage{todonotes}
\input epsf.tex
\usepackage{bbm}
\usepackage{epstopdf}
\usepackage{graphicx}
\usepackage{centernot}
\usepackage{mathtools}
\usepackage[all]{xy}
\usepackage{ stmaryrd }

\usepackage[capitalize,noabbrev]{cleveref}

\theoremstyle{plain}
\newtheorem{theorem}{Theorem}[section]
\newtheorem{proposition}[theorem]{Proposition}
\newtheorem{lemma}[theorem]{Lemma}
\newtheorem{corollary}[theorem]{Corollary}
\theoremstyle{definition}
\newtheorem{definition}[theorem]{Definition}

\theoremstyle{remark}
\newtheorem{remark}[theorem]{Remark}

\newcommand{\R}{\mathbb{R}}

\DeclareMathOperator{\fdim}{dim_\text{fun}}
\DeclareMathOperator{\rk}{rank}

\icmltitlerunning{Hidden Symmetries of Re{LU} Networks}

\begin{document}

\twocolumn[
\icmltitle{Hidden Symmetries of Re{LU} Networks}



\icmlsetsymbol{equal}{*}

\begin{icmlauthorlist}
\icmlauthor{J.~Elisenda Grigsby}{equal,bc}
\icmlauthor{Kathryn Lindsey}{equal,bc}
\icmlauthor{David Rolnick}{equal,mcgill,mila}
\end{icmlauthorlist}

\icmlaffiliation{bc}{Department of Mathematics, Boston College, Boston, USA}
\icmlaffiliation{mcgill}{School of Computer Science, McGill University, Montreal, Canada}
\icmlaffiliation{mila}{Mila -- Quebec AI Institute, Montreal, Canada}

\icmlcorrespondingauthor{J.~Elisenda Grigsby}{grigsbyj@bc.edu}

\icmlkeywords{deep learning theory, functional dimension, parameter space, linear region, activation pattern, bent hyperplane arrangement}

\vskip 0.3in
]



\printAffiliationsAndNotice{\icmlEqualContribution} 

\begin{abstract}
The parameter space for any fixed architecture of feedforward ReLU neural networks serves 
as a proxy during training for the associated class of functions -- but how faithful is this representation?  It is known that many different parameter settings $\theta$ can determine the same function $f$. Moreover, the degree of this redundancy is inhomogeneous: for some networks, the only symmetries are permutation of neurons in a layer and positive scaling of parameters at a neuron, while other networks admit additional \emph{hidden symmetries}. In this work, we prove that, for any network architecture where no layer is narrower than the input, there exist parameter settings with no hidden symmetries. We also describe a number of mechanisms through which hidden symmetries can arise, and empirically approximate the functional dimension of different network architectures at initialization. These experiments indicate that the probability that a network has no hidden symmetries decreases towards 0 as depth increases, while increasing towards 1 as width and input dimension increase.
\end{abstract}

\section{Introduction}
The success of deep learning relies upon the effectiveness of neural networks in expressing a wide variety of functions. However, it is generally impractical to explicitly write down the function computed by a network, so networks of a given architecture are described and learned via parameter vectors (encompassing weights and biases). The space of parameter vectors serves as a convenient proxy for the space of functions represented by a given network architecture, but it is an imperfect proxy since it is possible for two different parameter vectors to map to the same function.

Indeed, for any fully connected neural network with ReLU activation, it has been observed that the following transformations to the parameters are symmetries -- i.e., they do not change the function computed by the network \citep{RolnickKording, PhuongLampert}:

\begin{itemize}
\item \textbf{Permutation (P).} Reordering the neurons in any hidden layer, along with the corresponding permutation of the weights and biases associated with them,
\item \textbf{Scaling (S).} For any neuron in any hidden layer, multiplying the incoming weights and the bias by any $c > 0$, while dividing the outgoing weights by $c$.
\end{itemize}
Such symmetries can have important implications for gradient-based learning algorithms that operate on parameters. Many authors have considered methods to optimize neural networks accounting for scaling symmetries (see e.g.~\citet{neyshabur2015path}). While networks trained on the same data from different initializations are far apart in parameter space, they express similar functions;  recent work suggests that such networks may in fact be close in parameter space if one accounts for permutation symmetries (see e.g.~\citet{ainsworth2022git}).

It remains unknown, however, in what cases permutation and scaling are the only symmetries admitted by the parameters of a neural network, and how often there are other \emph{hidden symmetries} (formalization in Definition \ref{def:noHiddenSym}). \citet{RolnickKording} prove that under certain conditions, no hidden symmetries exist, and indeed that under these conditions it is possible to reverse-engineer a network's parameters up to permutation and scaling. \citet{PhuongLampert} prove that for all architectures with non-increasing widths, there exist parameter settings with no hidden symmetries. Work by \citet{GLMW} on the functional dimension of networks suggests that a wide variety of hidden symmetries may exist depending on the parameter setting.

Our key results in this paper are as follows:
\begin{itemize}
    \item We prove (Theorem \ref{thm:TPIC}) that if all layers in a fully connected ReLU network are at least as wide as the input layer, then there exists some setting of the parameters such that the network has no hidden symmetries. Indeed, we show that a positive-measure subset of parameter space admits no hidden symmetries.

   \item  We describe four mechanisms through which hidden symmetries can arise (Subsection \S \ref{ss:mechanisms}).  In particular, we prove (Proposition \ref{prop:containedInSubspace}) that if the image of the domain in a hidden layer is contained in a subspace of positive codimension, then there is ambiguity in the neuron of the next layer map. 
    
    \item We experimentally estimate the functional dimension of randomly initialized network {parameter settings}. Our results suggest that the probability that a network has no hidden symmetries decreases with depth, but increases as input dimension and width increase together.
\end{itemize}

\section{Related Work}

Several important lines of work have considered the symmetries of the parametric representations of deep ReLU networks and their implications for learning. One focus area has been in designing optimization methods for neural networks that are invariant to scaling symmetries at individual neurons. Approaches for achieving this goal have include path normalization \citep{neyshabur2015path}, manifold optimization \citep{badrinarayanan2015symmetry}, proceeding in a different vector space \citep{meng2019mathcal}, and projection onto a normalized manifold \citep{huang2020projection}.

Another fruitful direction of work has been in understanding how permutation symmetries in parameter space affect connectivity of the loss landscape. \citet{kuditipudi2019explaining} consider when different parameter permutations of a trained network are connected via piecewise linear paths in parameter space with low loss. \citet{brea2019weight} show that linearly interpolating between different permutations of a network leads to flat regions of the loss landscape. Several recent works have shown that, if permutation symmetries are taken into account, then it is possible to interpolate between networks trained from different initializations, while maintaining a low loss barrier \citep{entezari2021role,ainsworth2022git,jordan2022repair}.

In \citet{ArmentaJodoin}, the authors define and study the {\em moduli space} of neural network functions using quiver representation theory. This theory provides a framework for extracting global symmetries of parameter space of a network architecture from symmetries of the computational graph and the activation functions involved (see also \citet{GanevWalters}), \citet{armenta2023neural, SymmTeleport} build on these ideas to define {\em neural teleportation} algorithms aimed at using symmetries in the loss landscape to improve the efficiency of gradient descent in finding a minimal-loss solution. In \citet{NeuralMechanics} the authors argue that symmetries in the loss landscape have associated conserved quantities that impact training dynamics. 

A number of works have explicitly considered which symmetries are admitted by different ReLU networks. 
A number of authors consider the relationship between the parameters of a ReLU network and the geometry of its {\em bent hyperplane arrangement} (aka {\em fold set}); \citet{milli2019model}, \citet{RolnickKording}, and \citet{carlini2020cryptanalytic} use these properties to reverse-engineer the parameters of certain networks up to permutation and scaling symmetries, and \citet{PhuongLampert} proves that for certain architectures there exist parameter settings without hidden symmetries.
In particular, in \citet{RolnickKording}, the authors provide a geometric condition on the bent hyperplane arrangement of a parameter ensuring that the parameters can be reverse-engineered up to permutation and scaling, hence has no hidden symmetries. It follows nearly immediately that all depth 2 networks and a positive measure subset of any depth 3 network have no hidden symmetries. In \citet{PhuongLampert}, the authors prove that a  positive measure subset of parameters in every non-widening ($n_0 \geq n_1 \geq \ldots n_d$) architecture has no hidden symmetries. In the present work, we prove the complementary result that a positive measure subset of parameters in every architecture whose hidden layers are at least as wide as the input layer (that is, $n_0 \leq n_\ell$
 for all $\ell < d$) has no hidden symmetries.  An example of a family of architectures for which the question of the existence of parameters without hidden symmetries remains unresolved after the present work is an architecture of the form ($n_0, n_1, n_2, n_3, n_4$) with $n_0 < n_1$ and $n_2 <n_0$.
\color{black}

\citet{GLMW} study the {\em functional dimension} of a network parameter setting  -- the dimension of the space of functions that can be achieved by infinitesimally perturbing the parameters -- proving an upper bound on functional dimension that we conjecture is achieved for almost all parameter settings without hidden symmetries (cf.~Lemma \ref{lem:AchievesUpperBound}). 
\color{black}

\section{Notation and Background} \label{sec:background}
We consider fully connected neural networks with ReLU activation, denoting by $(n_0, \ldots, n_d)$ the architecture with input dimension $n_0$, hidden layer widths $n_1,n_2,\ldots,n_{d-1}$, and output dimension $n_d$.

Formally, let $\sigma: \mathbb{R}^n \rightarrow \mathbb{R}^n$ denote the function that applies the activation function $\mbox{ReLU}(x):= \max\{0,x\}$ component-wise. For an architecture $(n_0, \ldots, n_d)$, we define a parameter space $\Omega := \mathbb{R}^D$ where a parameter $\theta := (W^1,b^1, \ldots, W^d) \in \Omega$ consists of weight matrices $W^i \in \mathbb{R}^{n_{i+1} \times n_{i}}$ and bias vectors $b^i \in \mathbb{R}^{n_i}$ for $i=0, \ldots, d-1$. Accordingly, $D:=-n_d + \sum_{i=1}^{d} n_{i}(n_{i-1} + 1)$. From a parameter $\theta$ we define a neural network function: \begin{equation} \label{eqn:ReLUFunction} F_\theta: \xymatrix{\mathbb{R}^{n_0} \ar[r]^-{F^1} & \mathbb{R}^{n_1} \ar[r]^-{F^2} & \ldots \ar[r]^-{F^d} & \mathbb{R}^{n_d}},\end{equation} with layer maps given by: \begin{equation} \label{eqn:layermap}
F^i(x) := \left\{\begin{array}{cl}\sigma(W^ix + b^i) & \mbox{for $1 \leq i < d$}\\
W^ix  & \mbox{for $i = d$}.\end{array}\right.\end{equation}

Note that for any $\theta\in \Omega$, $F_\theta$ is a \emph{finite piecewise-linear} function -- that is, a continuous function for which the domain may be decomposed as the union of finitely many closed, convex pieces, on each of which the function is affine.

Following notation in  Definition 4 of \citet{Masden}, let  $F_{(\ell)} := F^\ell \circ \ldots \circ F^1$ denote the composition of layer maps from the domain, ending with the $\ell$th layer map,  and let $F^{(\ell)} := F^d \circ \ldots \circ F^\ell$ denote the composition of the layer maps ending at the codomain, beginning with the $\ell$th layer map.  In particular, $F_\theta = F^{(\ell+1)} \circ F_{(\ell)}$.

We refer to the components of $F_{(\ell)}$ as the {\em neurons} in the $\ell$th layer. The {\em pre-activation} map $z_{(\ell),i}: \mathbb{R}^{n_0} \rightarrow \mathbb{R}$ associated to the $i$th neuron in the $\ell$th layer is given by: \begin{equation} \label{eqn:preactneuron} z_{(\ell),i}(x) = \pi_i\left(W^\ell(F_{(\ell-1)}(x)) + b^\ell\right),\end{equation} where $\pi_i: \mathbb{R}^{n_\ell} \rightarrow \mathbb{R}$ denotes the projection onto the $i$th component. 
Following \citet{HaninRolnick}, we refer to the zero-set of the pre-activation map for the $i$th neuron in the $\ell$th layer as its associated {\em bent hyperplane}, $\hat{H}^{\ell}_i := z_{(\ell),i}^{-1}\{0\}$. 

The following notions from \citet{RolnickKording} (see also \citet{HaninRolnick}, \citet{GrigsbyLindsey}, and \citet{Masden}), will play a crucial role in the proofs of our results. A {\em (ternary) activation pattern} (aka {\em neural code} or {\em sign sequence}) for a network architecture $(n_0, \ldots, n_d)$ is an $N$--tuple $s \in \{-1,0,+1\}^N$ of signs for $N = \sum_{i=1}^d n_i$ (Def.~\ref{def:ternaryactivationpattern}). Each activation pattern determines a (frequently empty) subset of $\mathbb{R}^{n_0}$ called its associated {\em activation region}. Informally, an activation region is the collection of points $x$ for which the pre-activation sign of a non-input neuron matches the sign of the corresponding component of $s$ (Def.~\ref{def:activationregion}, or Def.~1 of \citet{HaninRolnick} and Def.~13 of \citet{Masden}.) A {\em linear region} of a finite piecewise linear function is a maximal connected set
on which the function is affine-linear. 
A ReLU network map $F_\theta: \mathbb{R}^{n_0} \rightarrow \mathbb{R}^{n_d}$ for an architecture $(n_0, \ldots, n_d)$ is said to satisfy the {\em Linear Regions Assumption} (LRA) (Def.~\ref{def:LRA}) if each linear region is the closure of a single non-empty activation region corresponding to a an activation pattern with all nonzero entries.

\citet{GrigsbyLindsey} proved that for almost all parameters in any fixed architecture $(n_0, \ldots, n_d)$, the {\em bent hyperplane} (zero set of the pre-activation output) associated to a non-input neuron has codimension $1$ (i.e., dimension $n_0 -1$) in the domain.\footnote{Note that it may also be empty.}

Moreover, it is proved in \citet{Masden} that for almost all parameters in any fixed architecture $(n_0, \ldots, n_d)$, the intersection of  $k$ bent hyperplanes has codimension $k$
in the domain. Following \citet{Masden}, we shall call a network whose bent hyperplanes satisfy this enhanced condition {\em supertransversal}.\footnote{The formal definition of supertransversality, given in Definition \ref{defn:supertransversal}, is stronger than what is stated here, but implies it.}

It is proved in Theorem 2 of \citet{RolnickKording} that if a supertransversal ReLU network map $F_\theta$ satisfies LRA,\footnote{Note that \cite{RolnickKording} do not need the LRA to be satisfied on the entire domain--only on a relevant subset  for their algorithm. See Section \ref{sec:LRA} in the Appendix.} and each pair of bent hyperplanes associated to each pair of neurons in each pair of adjacent layers has non-empty intersection, then the network admits no hidden symmetries. Indeed, the authors detail an algorithm allowing the parameters to be extracted from the local geometry of the intersections, up to permutation and scaling.

Accordingly, we will say that a network map $F_\theta$ associated to a parameter $\theta \in \Omega$ satisfies the {\em transverse pairwise-intersection condition}, abbreviated TPIC, if its associated bent hyperplane arrangement is supertransversal, and each pair of bent hyperplanes associated to each pair of neurons in each pair of adjacent layers has non-empty intersection. 
See Figure \ref{fig:2533}.

\color{black}

\section{Main Result}

\begin{theorem} \label{thm:TPIC} Let $(n_0, \ldots, n_d)$ be a feedforward ReLU network architecture satisfying $(n_0 = k) \leq n_\ell$ for all $\ell$, and let $\Omega$ denote its parameter space. A positive-measure subset of $\Omega$ has no hidden symmetries.
\end{theorem}

{\flushleft \bf Proof Sketch:} 
 A more explicit version of Theorem 2 of \citet{RolnickKording}, stated in Lemma \ref{lem:PreciseLRA}, tells us that any network map $F_\theta$ satisfying TPIC and LRA 
 on a neighborhood of the intersections admits no hidden symmetries.

By Proposition~\ref{p:TPICOpenCondition}, TPIC is an open condition. That is, any parameter satisfying TPIC has an open neighborhood of parameters also satisfying TPIC. Lemma \ref{lem:LRAinanhd} furthermore tells us that any parameter satisfying LRA in a neighborhood of the intersections has an open neighborhood on which LRA is satisfied in a neighborhood of the intersections.

We proceed by induction on the depth, $d$. When $d=1$, there is nothing to prove, so our true base case is $d=2$. In this case, we need only show that we can choose a  combinatorially stable parameter $\theta \in \Omega$ that satisfies (i) supertransversality, (ii) each bent hyperplane from $\ell = 2$ has non-empty intersection with each hyperplane from $\ell = 1$, and (iii) LRA on a neighborhood of the pairwise intersections. 

Since the set of parameters $\theta$ satisfying (i) and (iii) has full measure in $\Omega$, it is routine -- though technical -- to guarantee that these conditions are satisfied once we find a single parameter satisfying (ii). Details are given in the appendix.

To arrange that each bent hyperplane from $\ell=2$ has non-empty intersection with each hyperplane from $\ell = 1$, we make use of so-called {\em positive-axis hyperplanes}.  Generically, an affine hyperplane intersects each coordinate axis of $\mathbb{R}^n$ on either the positive or negative side. A {\em positive-axis} hyperplane intersects all coordinate axes on the positive side. Equivalently, a positive-axis hyperplane is describable as the zero set of an affine-linear equation with positive weights and negative bias: \[H := \{\vec{x} \in \mathbb{R}^n \,\,|\,\, \vec{w}\cdot \vec{x} + b = 0\}\] for $\vec{w} = (w_1, \ldots, w_n)$ satisfying $w_i > 0$ for all $i$ and $b < 0$.

The important property of a positive-axis hyperplane is that it has non-empty intersection with every origin-based ray contained in the non-negative orthant, $\mathbb{O}^{\tiny \geq 0} \subseteq \mathbb{R}^n$. 

We now use the fact (Lemma \ref{lem:intunbddcell}) that for almost all parameters $\theta$, the image under $F^1$ of almost every unbounded ray in $\mathbb{R}^{n_0}$ is a ray in $\mathbb{O}^{\tiny \geq 0} \subseteq \mathbb{R}^{n_1}$ (not necessarily based at the origin). Since every affine hyperplane in $\mathbb{R}^{n_0}$ contains an $(n_0 - 2)$--dimensional sphere of unbounded rays, we can ensure, by perturbing the parameters if necessary, that the image of every hyperplane $H^1 \subseteq \mathbb{R}^{n_0}$ associated to $F^1$ has non-empty intersection with any given positive-axis hyperplane $H^2 \subseteq \mathbb{R}^{n_1}$ with sufficiently high bias. We then apply the following lemma, with $G=F^1$, $H = F^2$, to each pair $S = F^1(H^1)$ for $H^1 \subseteq \mathbb{R}^{n_0}$ a hyperplane associated to the first layer map $F^1$ and $S' = H^2 \subseteq \mathbb{R}^{n_1}$ a sufficiently high bias positive-axis hyperplane associated to the second layer map $F^2$ to ensure that the the bent hyperplanes in the domain, $\mathbb{R}^{n_0}$, have non-empty pairwise intersection, as desired.

\color{black}
\begin{lemma} \label{lem:Nonemptypreimage} Let $G: A \rightarrow B$ and $H: B \rightarrow C$ be functions, let $F= H \circ G$ be their composition, and let $S, S' \subseteq B$ be subsets of the intermediate domain. Then $G^{-1}(S) \cap G^{-1}(S')$ is non-empty iff $G(A) \cap S \cap S'$ is non-empty.
\end{lemma}

\begin{proof} Immediate from the fact that \[G^{-1}(S) \cap G^{-1}(S') = G^{-1}\left(G(A) \cap S \cap S'\right).\]
\end{proof}

The inductive step in the construction of a combinatorially stable depth $d$ network from a  combinatorially stable depth $d-1$ network satisfying TPIC is more intricate, but the key idea is to notice that adding a layer to the network preserves all previous bent hyperplanes and their intersections, so all that is needed is to choose parameters for the final layer ensuring that the new bent hyperplanes associated to the final layer have non-empty pairwise intersection with the bent hyperplanes from the penultimate layer. 

This is where we will need to use one more key fact about the images of activation regions under ReLU neural network maps, which, when combined with Lemma \ref{lem:Nonemptypreimage} above, will allow us to conclude that certain points of intersection between the images of hyperplanes in the penultimate layer and hyperplanes associated to the final layer can be pulled back to obtain points of intersection between the corresponding bent hyperplanes in the domain:

\begin{proposition} \label{prop:AffIsomorph} Let $(n_0, \ldots, n_d)$ be a ReLU neural network architecture with $(n_0 = k) \leq n_{\ell}$ for all $\ell$. For almost all parameters $\theta \in \Omega$, if $C \subseteq \mathbb{R}^{k}$ is an activation region of $F_\theta$ with activation pattern $s_C = (s^1_C, \ldots, s^d_C)$, then $F_\theta(C)$ is a polyhedral set of dimension $\min\{\mbox{dim}(s^1_C), \ldots, \mbox{dim}(s^d_C), k\}$.
\end{proposition}

Here, $\mbox{dim}(s^\ell_C)$ refers to the number of $+1$'s in the binary tuple that forms the activation pattern $s^\ell_C$ associated to the output neurons of the $\ell$th layer map, $F^\ell$, on the activation region $C$ (cf.~Section \ref{sec:GeomComb}  in the Appendix). The above proposition has the following useful corollary, which allows us to conclude that any points of intersection between (bent) hyperplanes in the penultimate layer can be pulled back faithfully to points of intersection between the corresponding bent hyperplanes in the domain:

\begin{corollary} Let $(n_0, \ldots, n_d)$ be a ReLU network architecture and $\theta \in \Omega$ as above. If $C$ is an activation region of $F_\theta$ with activation pattern satisfying $\mbox{dim}(s^\ell_C) = k$ for all $\ell$, then $F_\theta$ restricted to the interior of $C$ is a homeomorphism onto its image. 
\end{corollary}

The linear regions assumption is satisfied for this construction because sufficiently many neurons from previous layers are active, which implies that we can distinguish activation regions (cf. Lemma 8 of \citet{HaninRolnick}).  The stability of this construction on a closed subset of the domain containing the pairwise intersections (Proposition \ref{prop:ternactivestable}) follows from genericity and supertransversality, which allows us to find a positive measure open neighborhood of $\theta$ that also admits no hidden symmetries.

\qed

See Figure \ref{fig:2533} for an illustration of what a bent hyperplane arrangement produced by our construction looks like for architecture $(n_0, \ldots, n_3) = (2,5,3,3)$.

\begin{figure}
\begin{center}
	\includegraphics[width=3in]{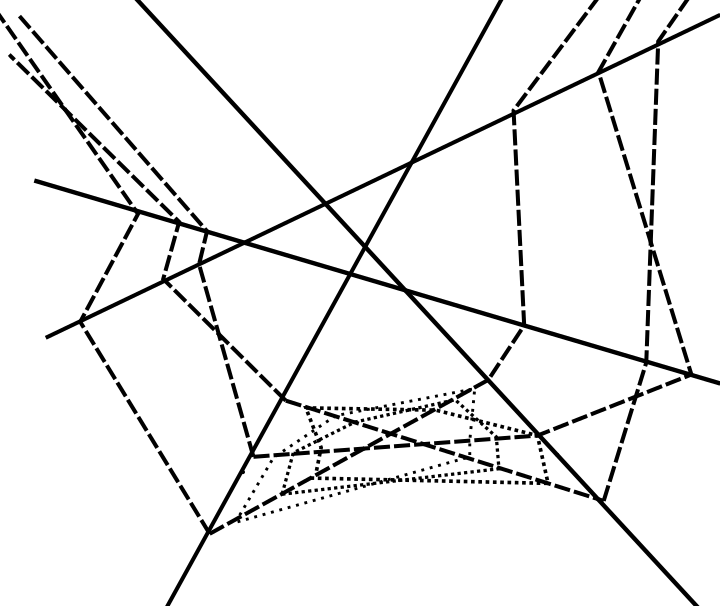}
	\caption{An illustration of a bent hyperplane arrangement in the domain $\mathbb{R}^{n_0}$ satisfying TPIC produced by our construction for a ReLU network of architecture (2,5,3,3). The bent hyperplanes for layers 2 (resp., 3) are represented with dark (resp., light) dashed lines.}
\label{fig:2533}
 \end{center}
\end{figure}

\section{Mechanisms by Which Hidden Symmetries Arise} \label{ss:mechanisms}

For any feedforward ReLU architecture $(n_0, \ldots, n_d)$ with at least one hidden layer (i.e. $d \geq 2$), the set of parameters admitting hidden symmetries has positive measure (\citet{GLMW}).
There are numerous mechanisms by which hidden symmetries can arise. We give a partial list below. 

\begin{enumerate}

\item \label{mech:stablyUnactivated} \emph{A stably unactivated neuron (Def.~\ref{def:stablyunactivated})}. 
The positive half-space in $\mathbb{R}^{n_\ell}$ associated to a neuron of the layer map $F^\ell$, for 
 $1 < \ell < d,$ could have empty intersection with $F_{(\ell-1)}(\mathbb{R}^{n_0})$, the image of $\mathbb{R}^{n_0}$ under the earlier layer maps. If this intersection remains empty under small perturbations of the parameters defining this neuron (while keeping the other neurons fixed), such perturbations will not alter the function. 

 In particular, the image of $\mathbb{R}^{n_0}$ in any hidden layer $\mathbb{R}^{n_{\ell}}$ is contained in the closed positive orthant, 
so a neuron whose associated half-space has (stably) empty intersection with the positive orthant results in a hidden symmetry.  See  Theorem 7.3 and Lemma 7.4 of \citet{GLMW}) for a probabilistic lower bound on this phenomenon, and Figure \ref{fig:StabUnactivated} for an illustration.  

\item \label{mech:nonCoActive} \emph{A pair of neurons in consecutive layers that are never co-active.} As noted in \citet{RolnickKording}, it is possible for two neurons in adjacent layers of a network to never be simultaneously active (cf. Definition \ref{def:noncoactive}). In such a case, the weight between these neurons is generically able to be perturbed without changing the function computed by the network. This is because either (a) the upstream neuron is inactive, in which case the downstream neuron receives zero input from it regardless of the weight between them, or (b) the downstream neuron is inactive, in which case it outputs zero regardless of the input received from the upstream neuron. See Figure \ref{fig:NeverCoActive}.

\item \label{mech:collapse} \emph{A ReLU of a later layer may collapse complexity constructed by earlier layers, negating the impact of the parameters from those layers}. One way this could happen is if multiple linear regions of $F^{i-1} \circ \ldots \circ F^1(\mathbb{R}^{n_0})$, $1 < i < d$ are collapsed by one or more ReLUs in layer $F^i$ to form a single linear region (violating LRA).  This collapsing could erase the effect of parameters from these earlier layers. See Figure \ref{fig:collapseComplexity}. 

\item \label{mech:subspace} The (relevant part of the) \emph{image in a hidden layer may be contained in a subspace of positive co-dimension.}  Suppose, for $1 < i < d$, the image of the domain in the $i$th hidden layer, $F^{i-1} \circ \ldots \circ F^1(\mathbb{R}^{n_0})$, is contained in an affine-linear subspace $A \subset \mathbb{R}^{n_i}$ of positive codimension.  Then for any neuron $N$ of the $(i+1)$th layer, denoting by $H$ the associated oriented and co-normed hyperplane in $\mathbb{R}^{n_i}$, there is a $1$-parameter family of oriented, co-normed hyperplanes $\{H_t\}$ obtained by rotating $H$ around its intersection with $A$ that all give rise to the same map (Lemma \ref{l:rotate}). See Figure \ref{f:1parameterFreeCorrected}.   

Proposition \ref{prop:containedInSubspace} says that, given a neuron map of the $(k+1)$th layer,  if the part of $$\textrm{Im}_{(k)} \coloneqq 
F^k \circ \ldots F^1(\mathbb{R}^{n_0})$$
where that neuron is nonnegative is contained in an affine-linear subspace of $\mathbb{R}^{n_k}$ of positive codimension, then the hyperplane associated to that is not uniquely determined.
     
\begin{proposition} \label{prop:containedInSubspace}
Let $\eta:\mathbb{R}^{n_k} \to \mathbb{R}^{1}$ be a neuron given by $\eta = \sigma \circ A$ for some fixed affine-linear map $A:\vec{x} \mapsto \vec{x}\cdot n_H - b_H$.     Suppose  
\begin{enumerate}
\item \label{i:lowdiml} there exists an affine-linear hyperplane $S \subset \mathbb{R}^{n_k}$  such that 
$\textrm{Im}_{(k)} \cap \{x \mid \eta(x) \geq 0\} \subseteq S,$
\item the hyperplanes $S$ and $H$ are in general position 
\item  \label{i:nonparallel} no unbounded ray contained in $\textrm{Im}_{(k)}$ has a slope that is orthogonal to $\vec{n}_H \neq 0$.

\end{enumerate}
Then there is a $1$-parameter family of (distinct) hyperplanes $\{H_t\}_t$ in $\mathbb{R}^{n_k}$ such that
$$\eta \circ F^k \circ \ldots F^1 = f_{H_t} \circ F^k \circ \ldots F^1$$ where $f_{H_t}: \mathbb{R}^{n_k} \to \mathbb{R}$ is a neuron associated to the hyperplane $H_t$ (equipped with some suitable choice of oriented conorm). 
\end{proposition}

The proof of Proposition \ref{prop:containedInSubspace} is in \ref{ss:mechanismsProofs}.  
\end{enumerate}

\begin{figure}[ht]
\labellist
\small\hair 2pt
\pinlabel  $H^\ell_1$ at 80 200
\pinlabel $F_{(\ell -1)}(\mathbb{R}^{n_0})$  at 260 145
\endlabellist
\begin{center}
\includegraphics[width=2.5in]
{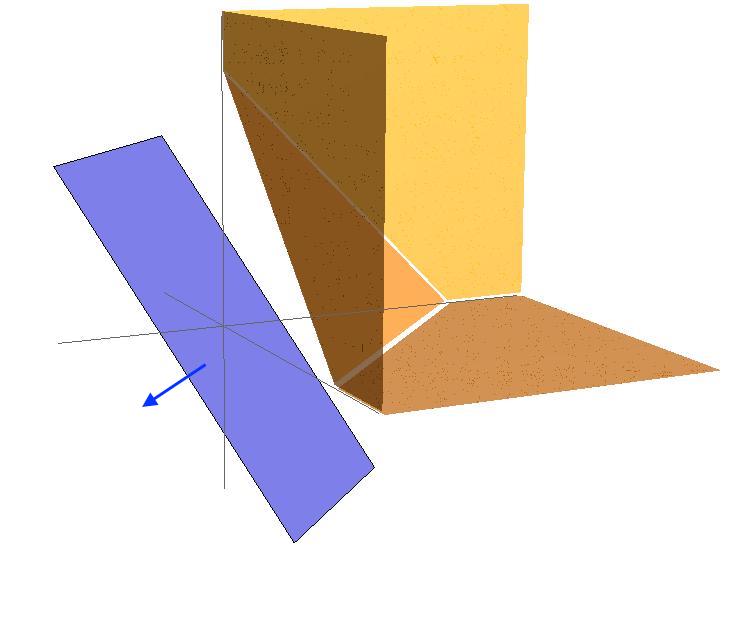}
\caption{Illustration of mechanism (i) through which hidden symmetries arise. In the figure above we see $\mathbb{R}^{n_{\ell-1}}$, the domain of the hidden layer map, $F^\ell$. The image (represented in shades of orange) of $\mathbb{R}^{n_0}$ under the composition, $F_{(\ell-1)}$, of all previous layer maps is always contained in the non-negative orthant. Since the positive half-space of $H_1^\ell$ misses the non-negative orthant, any sufficiently small perturbation of $H_1^{\ell}$ will have no impact on the overall function.}
\label{fig:StabUnactivated}
 \end{center}
\end{figure}

\begin{figure}[ht]
\labellist
\small\hair 2pt
\pinlabel  $\hat{H}^\ell_i$ at 230 450 
\pinlabel $\hat{H}_j^{\ell -1}$  at 100 570
\endlabellist
\begin{center}
\includegraphics[width=2in]{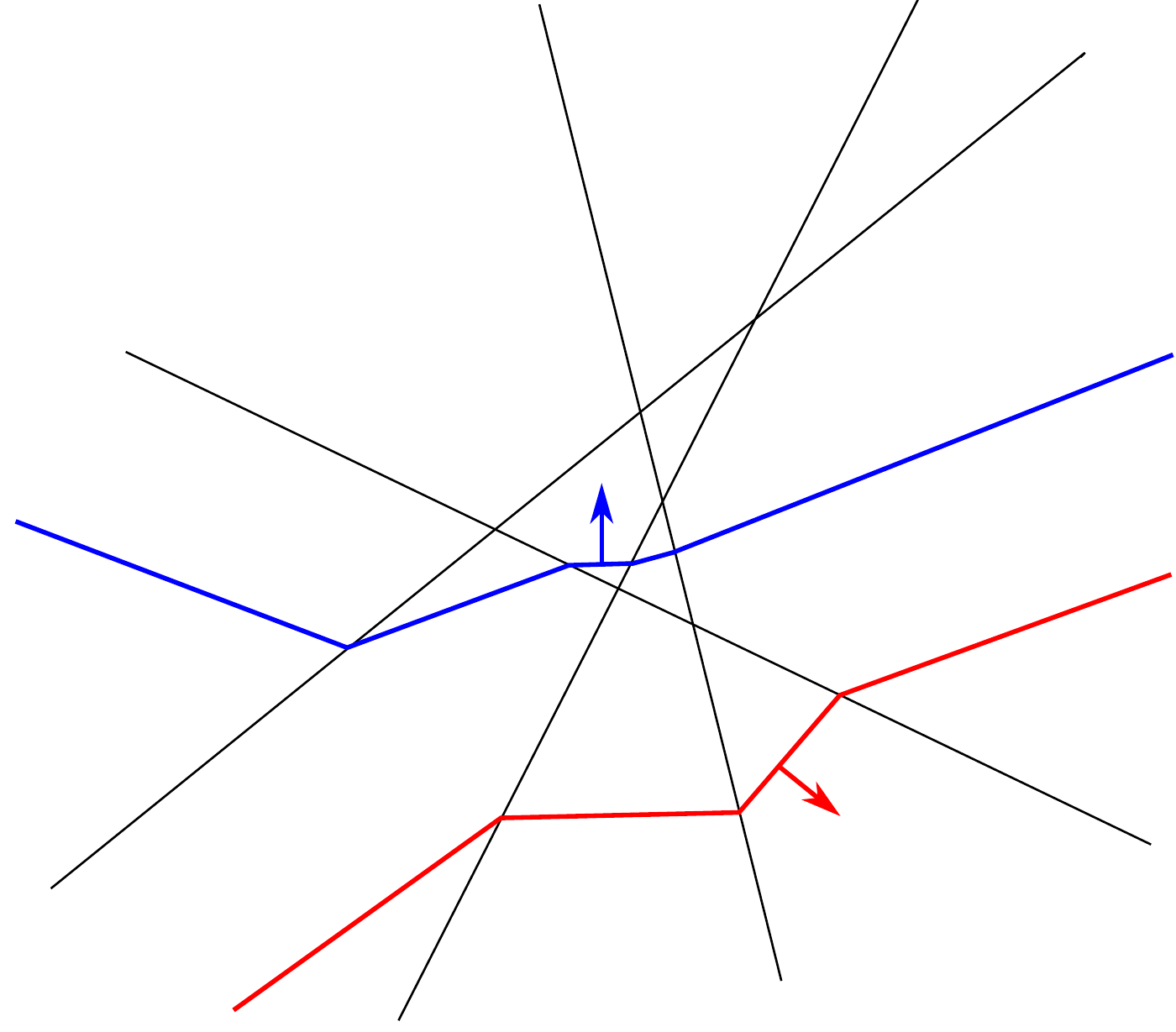}
\caption{Illustration of mechanism (ii) through which hidden symmetries arise. In the figure above we see two non-intersecting bent hyperplanes from adjacent layers, pulled back to $\mathbb{R}^{n_0}$, the input layer. Because their intersection is empty, they admit a co-orientation (pictured) for which they are never co-active. It follows that the overall function has no dependence on the weight $W_{ij}^\ell$ between them.}
\label{fig:NeverCoActive}
 \end{center}
\end{figure}

\begin{figure}[ht]
\labellist
\small\hair 2pt
\pinlabel  $F_{(\ell -1)}(\mathbb{R}^{n_0})$ at 80 100
\pinlabel $H_1^{\ell}$  at 200 220
\endlabellist
\begin{center}
\includegraphics[width=2in]
{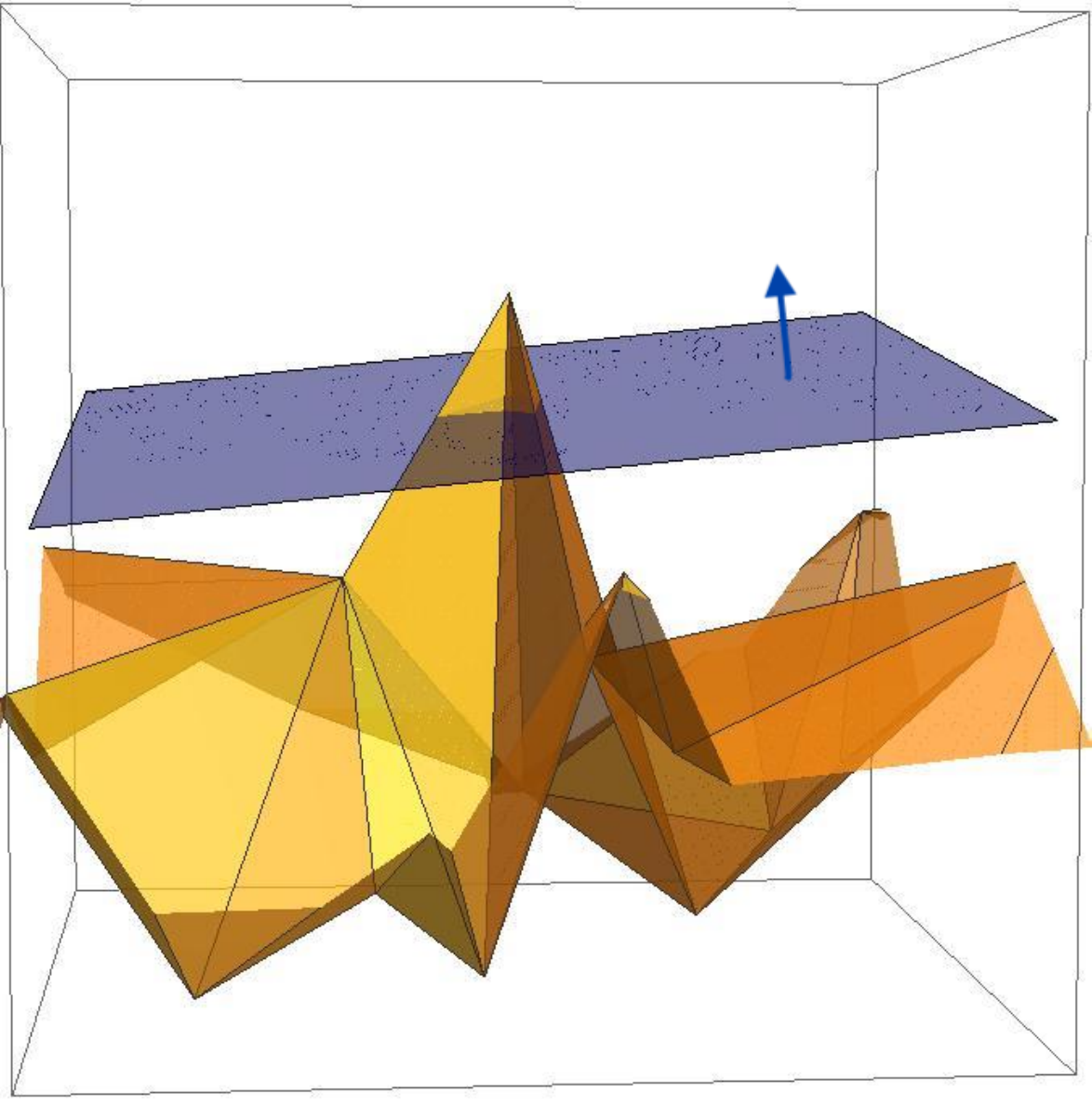}
\caption{Illustration of mechanism \eqref{mech:collapse} through which hidden symmetries arise.  The orange surface is the image of the domain inside $\mathbb{R}^{n_{\ell -1}}$. The blue hyperplane $H_1^{\ell}$, which is oriented upwards, is associated to a neuron of $F^{\ell}$; everything below the blue plane is zeroed out by the ReLU of this neuron. This neuron is not stably unactivated, because the orange surface has nonempty intersection with the positive half space above the blue plane. However, perturbing parameters that are only expressed in the part of the orange surface below the blue plane does not change the output of the neuron.} 
\label{fig:collapseComplexity}
 \end{center}
\end{figure}

 \begin{figure}[ht]
 \labellist
\small\hair 2pt
\pinlabel  $S$ at 130 130
\pinlabel $H$  at 290 360
\pinlabel $H_t$  at 190 250
\endlabellist
\centering
\includegraphics[width=4.5cm]{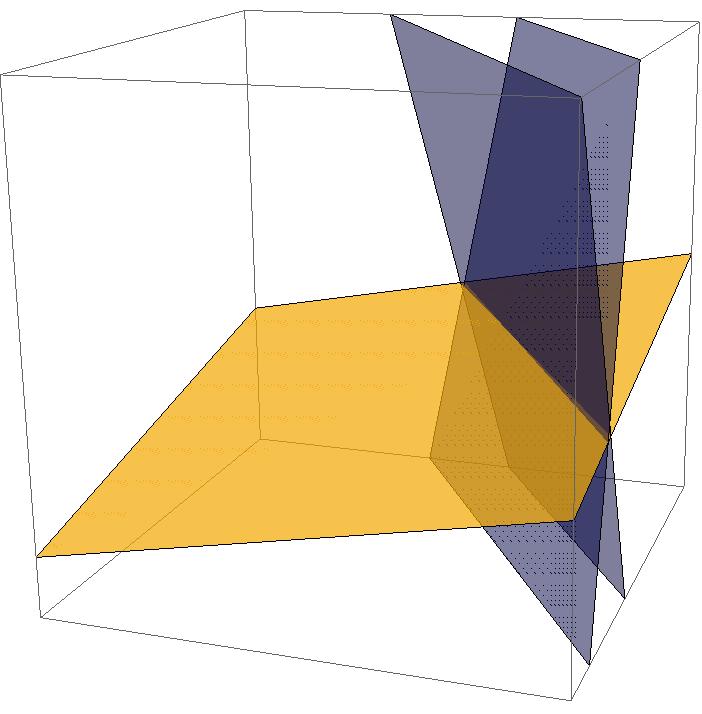}
\caption{Illustration of mechanism (iv) through which hidden symmetries arise. If the image of $\mathbb{R}^{n_0}$ in a hidden layer $\mathbb{R}^{n_\ell}$ is contained in the affine-linear subspace $S$, replacing the neuron $N$ defined by hyperplane $H$ with the neuron $N_t$ defined by hyperplane $H_t$ (with a suitably scaled co-norm) yields the same function. 
}
\label{f:1parameterFreeCorrected}
\end{figure}

\begin{figure*}[ht]
\begin{center}
\includegraphics[width=0.75\linewidth]{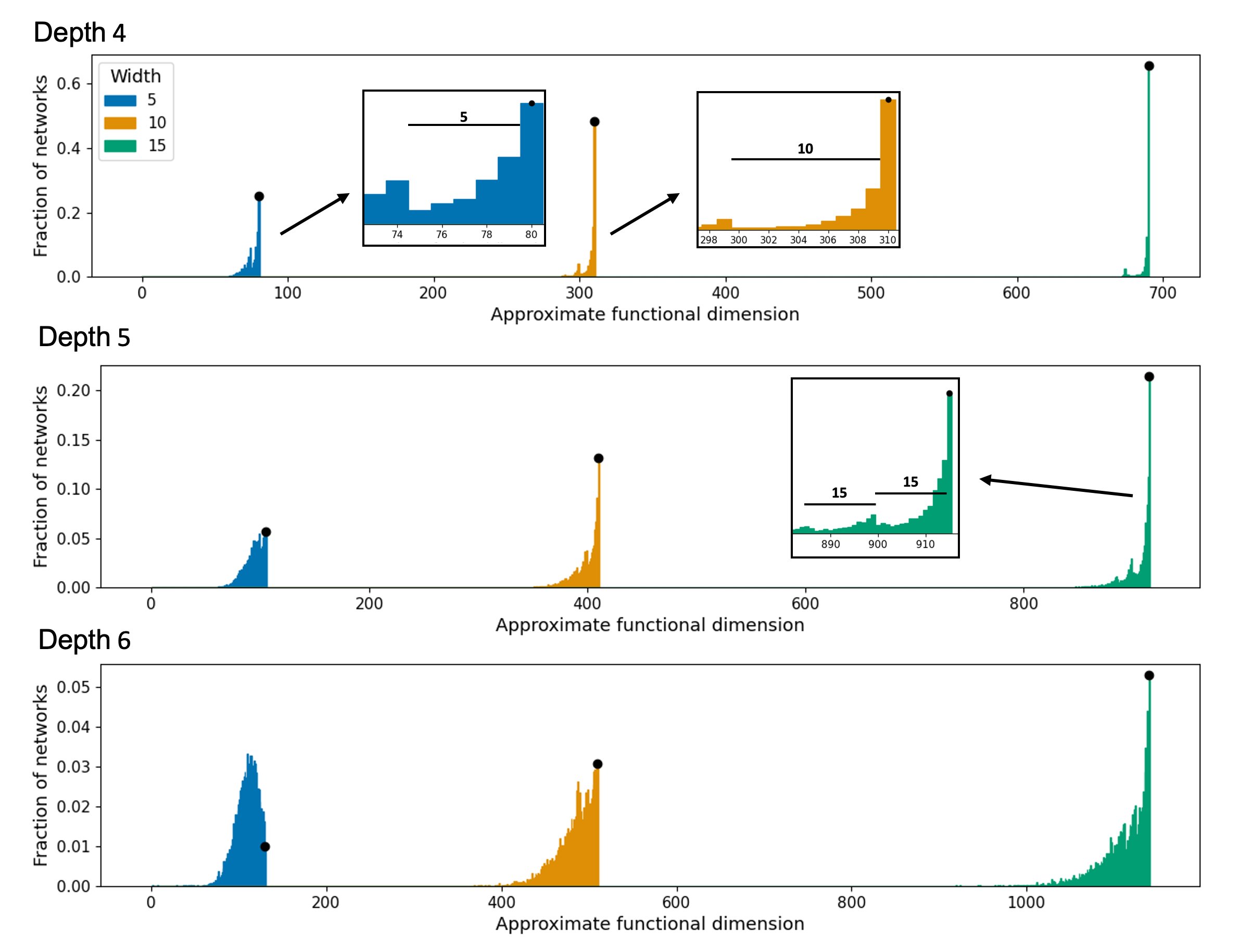}
\caption{For each of various network architectures, we approximate the distribution of functional dimensions as the parameter $\theta\in \Omega$ varies. Different plots show networks of different depths, while different colors in a plot correspond to varying the input dimension and width $n_0=n_1=\cdots = n_{d-1}$. Black dots show the fraction of networks with full functional dimension (no hidden symmetries). We observe that the proportion of networks with full functional dimension increases with width and decreases with depth. Insets in the figure zoom in on certain multimodal distributions, showing that modes are spaced apart by approximately the width of the network.}
\label{fig:main}
\end{center}
\end{figure*}

\color{black}

\section{Experiments} \label{sec:experiments}

We conducted an empirical investigation of hidden symmetries at various parameter settings. It is computationally impractical to directly rule out the existence of hidden symmetries at a parameter using, e.g., the geometry of the bent hyperplane arrangement. Accordingly, in our experiments we rely on a relationship between the symmetries of a parameter and its {\em functional dimension} to probe the prevalence of hidden symmetries for a variety of architectures. 

Recall that the functional dimension $\fdim(\theta)$ of parameter $\theta$ is, informally, the number of linearly independent ways the function $F_{\theta}$ can be altered by perturbing the parameter $\theta$. Following the formal Definition \ref{def:functionaldimension}, we can approximate $\fdim(\theta)$ by evaluating the function $F_\theta$ at a finite subset $Z\subset \R^{n_0}$ of points in input space, stacking the outputs into a single vector of dimension $|Z|\cdot n_d$, and calculating the rank of the Jacobian of this vector with respect to $\theta$. The functional dimension $\fdim(\theta)$ is the supremum of this rank over all sets $Z$. Intuitively, this is because we are evaluating the number of coordinates of $\theta$ that independently affect the value of the function $F_\theta$ at some point in its domain (recognizing that some weights and biases may not have an effect on $F_\theta$ except on a limited subset of input space).

In our experiments, we consider networks with $n_d=1$, and to approximate the functional dimension, we evaluate the set of gradient vectors $\{\nabla_\theta F_\theta(z)\}_{z\in Z}$ over a finite subset of $m$ points $Z=\{z_1,\ldots,z_m\}\subset \R^{n_0}$ in input space (we use points sampled i.i.d.~from the zero-centered unit normal). Then, for $m$ sufficiently large, we have:
\begin{equation}\fdim(\theta) \approx \rk\left(\left[\begin{array}{c} \nabla_\theta F_\theta(z_1) \\ \vdots \\ \nabla_\theta F_\theta(z_m)\end{array}\right]\right).\label{eq:rank}\end{equation}

We initialize networks with weights drawn i.i.d.~from normal distributions with variance $2/\text{fan-in}$, according to standard practice for ReLU networks \citep{he2015delving, hanin2018start}, and biases drawn i.i.d.~from a normal distribution with very small variance (arbitrarily set to $0.01$). To improve the quality of the approximation in \eqref{eq:rank}, we use $m$ sample points for $m$ equal to 100 times the maximum possible value for $\fdim(\theta)$, according to the upper bound given in \citet{GLMW} (in Appendix \ref{app:experiments}, we show that our experimental conclusions are not dependent on the choice of $m$). Note that our approach yields approximations of $\fdim(\theta)$ which are also necessarily lower bounds to it; in particular, any computed value that attains the theoretical upper bound on $\fdim(\theta)$ is guaranteed to be accurate and thus is consistent with the parameter admitting no hidden symmetries.

In Figure \ref{fig:main}, we plot the distribution of approximate functional dimensions for networks with depth $d=4,5,6$ and with $n_0=n_1=\cdots=n_{d-1}$ equal to $5,10,15$. For each architecture, we consider $5000$ different choices of $\theta\in \Omega$, computing the fraction that lead to networks with a given approximated functional dimension. Thus, we find that for depth 4, the widths $5,10,15$ result in respectively 25\%, 48\%, 66\% percent of networks having the maximum possible functional dimension (marked with black dots in the Figure), while for depth 6, the widths $5,10,15$ result in respectively 1\%, 3\%, 5\% percent of networks having the maximum possible functional dimension.

We observe that increasing the depth $d$ (while keeping the width fixed) results in a decreased probability of full functional dimension, and thus the likely absence of hidden symmetries (cf.~Lemma \ref{lem:AchievesUpperBound}). By contrast, increased width (with $n_0=n_1=\cdots=n_{d-1}$, i.e. varying the input dimension and width together, while  keeping depth fixed) is associated with an increasing probability of full functional dimension. 

We offer the following explanations of possible drivers of these observed phenomena. Increasing the depth increases the variance and higher moments associated with properties such as the activation of individual neurons \citep{hanin2018start, hanin2018neural}, thereby increasing the likelihood that functional dimension is decreased via mechanisms \eqref{mech:stablyUnactivated} or \eqref{mech:collapse}. By contrast, increasing the input dimension and width increases the chance that the bent hyperplanes associated with two neighboring neurons will intersect. (Intuitively, this is because a hyperplane will intersect a bent hyperplane unless the latter ``curves away'' in all dimensions, which becomes exponentially unlikely as the dimension increases, in the same way that the probability a matrix is positive definite decays exponentially with dimension \citep{dauphin2014identifying}). However, further study is required, and it is worth noting that a counteracting factor as the width increases may be that the maximum functional dimension increases, so the support of the distribution also increases and the probability assigned to any individual functional dimension, including the maximum, is less than it might be for a distribution with smaller support.

 We also note that the distributions of approximate functional dimensions appear to approach smooth unimodal curves if the probability of full functional dimension is low (as in the Depth 6 plots), but are strongly multimodal when there is a high probability of full functional dimension.  In the inset panels of the figure, we show zoomed-in versions of the upper ends of certain distributions, detailing the multiple peaks in the distribution.  We note that for each such multimodal distribution, the peaks appear to be spaced by a value equal to the width of the network. We note that of the mechanisms we consider for hidden symmetries, mechanism \eqref{mech:stablyUnactivated} (a stably inactivated neuron) should reduce functional dimension by $2\times \text{width}$ (the number of incoming and outgoing weights of the neuron), while \eqref{mech:nonCoActive} (two neurons that are never co-active) reduces functional dimension by one (the weight between the neurons). Thus, neither of these mechanisms should apply in this case, and mechanisms \eqref{mech:collapse} or \eqref{mech:subspace} may apply; the phenomenon bears further investigation. 

In Figure \ref{fig:varying_width}, we show the results of a similar set of experiments, where the input dimension $n_0$ is instead kept fixed at 5 as the widths $n_1=n_2=\cdots=n_{d-1}$ of the hidden layers vary.
While we again observe in this case that the probability of full functional dimension decreases with depth, this probability slightly decreases with width, unlike the previous scenario (Figure \ref{fig:main}).

As in Figure \ref{fig:main}, we again note that when the distributions are multimodal, the gaps between the modes are spaced according to the width $n_1=\cdots=n_{d-1}$. 

\begin{figure*}[ht]
\begin{center}
\includegraphics[width=0.75\linewidth]{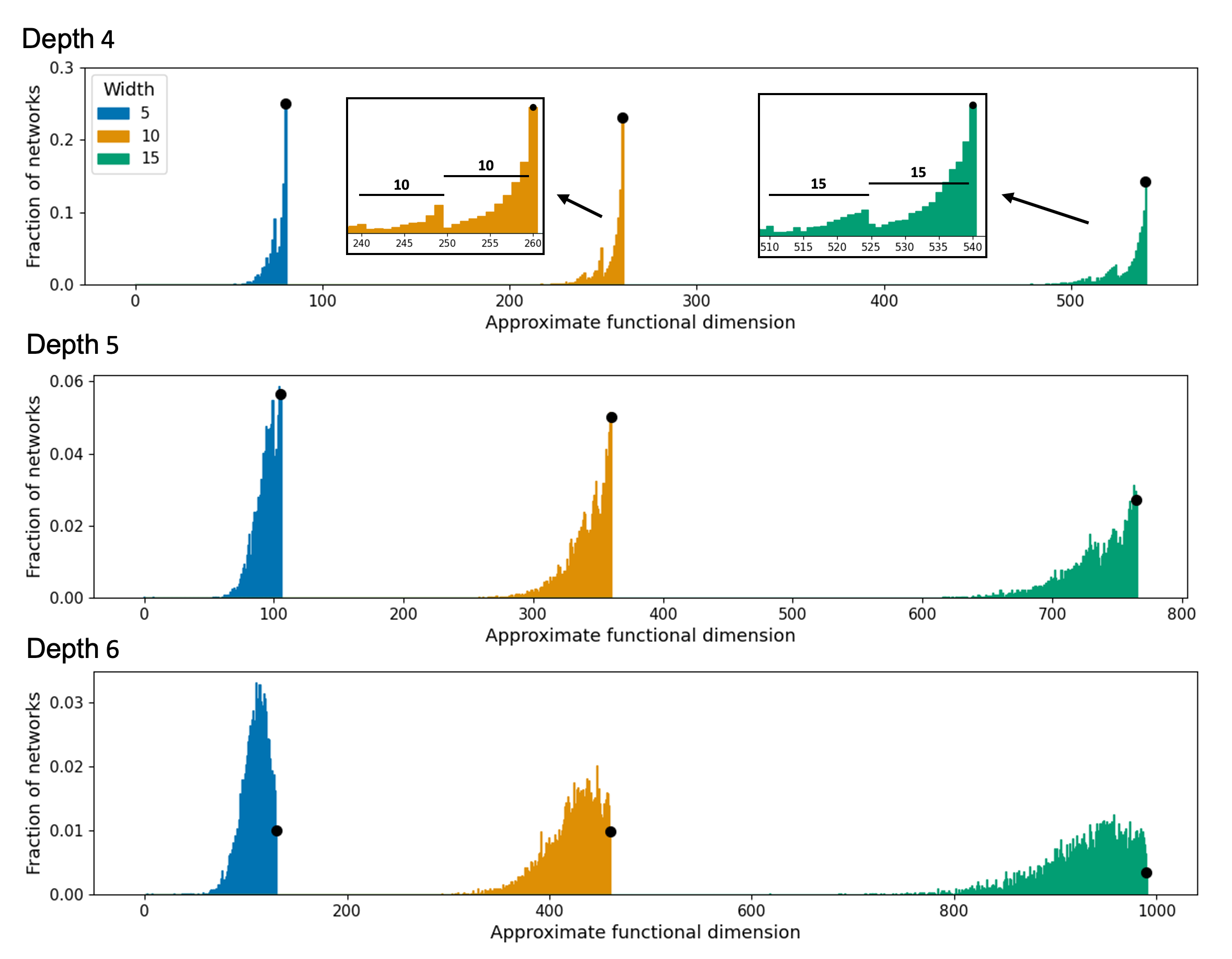}
\caption{This figure shows similar experiments as in Figure \ref{fig:main}, but with input dimension fixed at 5 instead of equal to the width of hidden layers. Here, we observe that the proportion of networks with full functional dimension decreases with depth as in Figure \ref{fig:main}, while slightly decreasing with width.}
\label{fig:varying_width}
\end{center}
\end{figure*}
\color{black}

\section{Conclusions and Further Questions} 

We have performed both a theoretical and empirical investigation of the following question: {\em How faithfully does the parameter space of a feedforward ReLU network architectures model its associated function class?} 

Our investigation centers on a relationship between well-established {\em symmetries} of parameter space (operations on parameters that leave the resulting function unchanged) 
and the {\em functional dimension} of a parameter (informally, the true dimension of the local search space for any gradient-based optimization algorithm). It was established in \citet{GLMW} that the functional dimension is inhomogeneous across parameter space, but the prevalence of this inhomogeneity and specifics about its dependence on architecture was previously unknown. 

In the theoretical component of this work, we significantly expand the collection of architectures containing parameters with no hidden symmetries beyond the restricted classes considered in \citet{PhuongLampert} and \citet{RolnickKording}. We also provide a partial list of geometric mechanisms  that give rise to positive-dimensional spaces of hidden symmetries.

 Our empirical investigation strongly suggests that the probability distribution on the functional dimension at initialization is both interesting and architecture-dependent. In particular, under standard assumptions on the probability distribution on the parameters, the expected value of the functional dimension appears to scale positively with width and negatively with depth. It also appears to be multimodal when the ratio of the width to the depth is high, with modes separated by integer multiples of the width. Further investigation of these effects 
 may help us understand which mechanisms dominate in producing hidden symmetries, at various depth vs.~width scales.

 In future work, we hope to investigate how functional dimension evolves during training, since parameters with lower functional dimension are associated to lower-complexity functions that are more likely to generalize well to unseen data. Since lower functional dimension corresponds to higher-dimensional spaces of local symmetries, low functional dimension should induce local flatness of the loss landscape. Comparing this conjecture to recent work suggesting that stochastic gradient descent favors flat minima of the loss landscape,\footnote{Critical points for which the Hessian of the loss has many eigenvalues close to $0$.} this could at least partially explain any implicit regularization behavior  of SGD for feedforward ReLU neural network architectures.

\section*{Acknowledgments}
J.E.G. ~acknowledges support from Simons Collaboration grant 635578 and NSF grant  DMS - 2133822.  K.L. ~acknowledges support from NSF grants DMS-2133822 and DMS-1901247.
D.R.~acknowledges support from the Canada CIFAR AI Chairs Program and an NSERC Discovery Grants.

\bibliography{dimensionbibliography}
\bibliographystyle{icml2023}

\newpage
\appendix
\onecolumn

\section{Combinatorial/Geometric Background and Notation}
\label{sec:GeomComb}

Let \[\mathbb{O}^{\geq 0}:= \{(x_1, \ldots, x_n) \in \mathbb{R}^n \,\,|\,\, x_i \geq 0 \,\,\forall\,\, i\}\] denote the non-negative orthant in $\mathbb{R}^n$. Letting \[\mathbb{R}^n_k := \{(x_1, \ldots, x_n) \in \mathbb{R}^n\,\,|\,\, x_{k+1} = \ldots x_n = 0\}\] denote the initial coordinate $k$--plane, we will denote by \[\mathbb{O}^{\geq 0}_k := \mathbb{O}^{\geq 0} \cap \mathbb{R}^n_k\] the distinguished $k$-face of the non-negative orthant which is obtained by intersecting with the initial coordinate $k$--plane.

 Let $\theta \in \Omega$ be a parameter in the parameter space of a feedforward ReLU network architecture $(n_0, \ldots, n_d)$, and let $F_\theta$ be defined as in Equations \ref{eqn:ReLUFunction} and \ref{eqn:layermap}. 

 Letting $z^\ell_i = \pi_i (W^\ell x + b^\ell)$ denote the $i$th component of the pre-activation output of the $\ell$th layer map $F^\ell:\mathbb{R}^{n_{\ell-1}} \to \mathbb{R}^{n_\ell}$ of $F_\theta$, we denote its zero set by $H^\ell_i := (z^\ell_i)^{-1}\{0\} \subseteq \mathbb{R}^{n_{\ell - 1}}.$ Note that for almost all parameters, $H^{\ell}_i$ is an affine hyperplane. 
 
 Accordingly, we associate to each layer map $F^\ell$ the set \[\mathcal{A}^{\ell} = \{H^\ell_1, \ldots, H^\ell_{n_\ell}\} \subseteq \mathbb{R}^{n_{\ell-1}},\] which for almost all parameters is a hyperplane arrangement. In the course of the inductive proof of our main theorem, we will need notation for the preimages of the hyperplanes in the previous layer:  
\[\grave{\mathcal{A}}^{\ell} = \{\grave{H}^{\ell}_i\}_{i=1}^{n_\ell} := \left\{F_\ell^{-1}\left(H^{\ell}_i\right)\right\}_{i=1}^{n_\ell} \subseteq \mathbb{R}^{n_{\ell-2}}.\]
and in the domain (the reader easily checks that the latter are precisely the bent hyperplanes defined in Equation \ref{eqn:preactneuron})\color{black}:
\[\hat{\mathcal{A}}^\ell = \{\hat{H}^\ell_i\}_{i=1}^{n_\ell} := \left\{F_{(\ell)}^{-1}(H_i^\ell)\right\}_{i=1}^{n_\ell} \subseteq \mathbb{R}^{n_0}\]

 $\mathcal{A}^\ell$ is said to be \emph{generic} if for all subsets \[\{H^\ell_{i_1} , \ldots , H^\ell_{i_p}\} \subseteq \mathcal{A}^\ell,\] it is the case that $H^\ell_{i_1} \cap \ldots \cap H^\ell_{i_p}$ is an affine-linear subspace of $\mathbb{R}^{n_{\ell-1}}$ of dimension $n_{\ell-1} - p$, where a negative-dimensional intersection is understood to be empty. 
 
 A layer map $F^\ell$ is said to be \emph{generic} if $\mathcal{A}^\ell$ is generic.  A parameter $\theta$ or the corresponding network map $F_\theta$ is said to be generic if all of its layer maps are generic.

 It is well-established in the hyperplane arrangement literature (cf.~\citet{Stanley}) \color{black} that generic arrangements are full measure. It follows \citep{GrigsbyLindsey} that generic network maps are full measure in parameter space.

\subsection{Decompositions of Polyhedral Sets}
Recalling that a {\em polyhedral set} in $\mathbb{R}^{n_{\ell -1}}$ is an intersection of finitely many closed half spaces, a hyperplane arrangement in $\mathbb{R}^{n_{\ell-1}}$ induces a {\em polyhedral decomposition} of $\mathbb{R}^{n_{\ell -1}}$ into finitely many polyhedral sets. The face structure on these polyhedral sets gives the decomposition the structure of a polyhedral complex. By pulling back these polyhedral complexes to the domain, $\mathbb{R}^{n_0}$, and taking intersections, we inductively obtain the canonical polyhedral complex $\mathcal{C}(F_\theta)$ as follows. 
 
 For $\ell \in \{1,\ldots,d\}$, denote by $R^{\ell}$ the polyhedral complex on $\mathbb{R}^{n_{\ell-1}}$ induced by the hyperplane arrangement associated to the $\ell^{\textrm{th}}$ layer map, $F^{\ell}$.  Inductively define polyhedral complexes $\mathcal{C}\left(F_{(1)}\right),\ldots,\mathcal{C}\left(F_{(\ell)}\right)$ on $\mathbb{R}^{n_0}$ as follows: Set $\mathcal{C}\left(F_{(1)} = F^{1}\right):= R^{1}$ and for $i = 2,\ldots,m$, set
\begin{equation*}
\mathcal{C}(F_{(\ell)})
:= \left \{S \cap \left(F_{(\ell)}\right)^{-1}(Y) \mid S \in \mathcal{C}\left(F_{(\ell)}\right), Y \in R^{\ell} \right \}.
 \end{equation*}
Set $\mathcal{C}(F_\theta) := \mathcal{C}(F_\theta = F_{(d)})$.  See \citet{GrigsbyLindsey} and \citet{Masden} for more details.

It was proved in \citet{GrigsbyLindsey} that on a full measure subset of parameter space, the $n_0$-cells of the canonical polyhedral complex, $\mathcal{C}(F)$, are the closures of the activation regions, and the $(n_0-1)$-skeleton of $\mathcal{C}(F)$ is the bent hyperplane arrangement associated to $F_\theta$. 

We will also need the following terminology and results (cf. \citet{Schrijver}) pertaining to the structure of polyhedral sets. See \citet{GrigsbyLindsey} and \citet{GLMas} for additional details.

A polyhedral set $P \subset \mathbb{R}^n$ is said to be {\em  pointed} if it has a face of dimension $0$. The \emph{convex hull} of a set $S \subset \mathbb{R}^n$ is the intersection of all convex subsets of $\mathbb{R}^n$ that contain $S$. A \emph{cone} in $\mathbb{R}^n$ is a set $C$ such that if $x,y \in C$ and $\lambda, \mu \geq 0$, then $\lambda x + \mu y \in C$.  Let $P$ and $Q$ be polyhedral sets embedded in $\mathbb{R}^n$.  The \emph{Minkowski sum} of $P$ and $Q$ is $P + Q := \{p+q \mid p \in P, q \in P\}.$ The \emph{characteristic cone} of $P$, denoted \textrm{Cone}(P), is the maximal set $\{ y \mid x+y \in P \textrm{ for all }x \in P\}$ that also has the structure of a cone. Note that a polyhedral set is unbounded iff its characteristic cone is non-empty. Moreover, every polyhedral set $P$ has a decomposition as the Minkowski sum of a bounded polyhedral set (polytope) $P_B$ and $\mbox{Cone}(P)$. If $P$ is pointed, we may take $P_B$ to be the convex hull of its $0$--cells, cf. Theorem 8.5 of \citet{Schrijver}.

\begin{definition} \label{def:ternaryactivationpattern}
A {\em ternary activation pattern} (aka {\em ternary neural code} or {\em ternary sign sequence}) for a network architecture $(n_0, \ldots, n_d)$ with $N$ neurons is a ternary tuple $s \in \{-1,0,+1\}^N$. The \emph{ternary labeling} of a point $x \in \mathbb{R}^{n_0}$ is the sequence of ternary tuples $$s_x := \left(s_x^{1}, \ldots, s_x^{d}\right) \in \{-1,0,+1\}^{n_1 +\ldots + n_d}$$ indicating the sign of the pre-activation output of each neuron of $F_\theta$ at $x$. \end{definition} 

Explicitly, letting $F_\theta$ be defined as in Equations \ref{eqn:ReLUFunction} and \ref{eqn:layermap}, and $x \in \mathbb{R}^{n_0}$ any input vector, and the pre-activation output $z_{(\ell),i}(x)$ of the $i$th neuron in the $\ell$th layer at $x$ is as in Equation \ref{eqn:preactneuron}, $s_x^{\ell} = \left(s^{\ell}_1, \,\, \ldots \,\,, s^{\ell}_{n_\ell}\right)$ are defined by $s^{\ell}_{x,i} = \mbox{sgn}(z_{(\ell,i)}(x))$ (using the convention $\textrm{sgn}(0) = 0$). 

Moreover, for all parameters $\theta$ it follows immediately from the definitions that the ternary labeling is constant on the interior of each cell of $\mathcal{C}(F_\theta)$, inducing a ternary labeling $s_C$ on each cell $C$ of $\mathcal{C}(F_\theta)$ \citep{Masden}.  If $s^{\ell}_{x,i} \leq 0$ at an input vector $x$ (resp., $s^{\ell}_{C,i} \leq 0$ on a cell $C$), we say that the $i$th neuron in the $\ell$th layer is \emph{off} or \emph{turned off} at $x$ (resp., on $C$).

\begin{definition}[\citet{GLMW}] \label{def:stablyunactivated}
Fix a parameter $\theta \in \Omega$. A neuron (say it is the $i$th neuron of layer $\ell$) is said to be \emph{stably unactivated} at $\theta$ if there exists an open neighborhood $U \subset \Omega$ of $\theta$ such $s^{\ell}_{x,i}(u) \leq 0$ for every $x \in \mathbb{R}^{n_0}$ and every $u \in U$, where
$s^{\ell}_{x,i}(u)$ denotes the corresponding coordinate of ternary coding \emph{with respect to the parameter $u$}. 
\end{definition}

\begin{definition} \label{def:activationregion}
The \emph{activation region} of $F_\theta$ corresponding to a ternary activation pattern $s$ is a maximal connected component of the set of input vectors $x \in \mathbb{R}^{n_0}$ for which the ternary labeling $s_x$ equals $s$.
\end{definition}

\begin{definition} \label{Def:pmactivationregion}
A \emph{$\pm$-activation pattern} is a ternary activation pattern in which every coordinate is nonzero, and a \emph{$\pm$-activation region} is an activation region associated to a $\pm$-activation pattern. 
\end{definition}

 Note that any $\pm$-activation region is an open set.

\begin{definition}\label{def:noncoactive}
Neuron $i$ of layer $\ell$ and neuron $j$ of layer $\ell+1$ are called \emph{never coactive} if
$\{x \in \mathbb{R}^{n_0} \mid s^{\ell}_{x,i} = s^{\ell+1}_{x,j} = 1\} = \emptyset$. 
\end{definition}

\begin{remark} \label{rem:GenSupTransCase1skeleton}
For generic, supertransversal  (Definition \ref{defn:supertransversal}) networks, it follows from \citet{GrigsbyLindsey} that the  $\pm$-activation regions of $\mathcal{C}(F_\theta)$ are precisely the interiors of the $n_0$--cells of $\mathcal{C}(F_\theta)$. 
\end{remark}

\color{black}

\begin{definition} For $s = (s^1, \ldots, s^d) \in \{-1,0,+1\}^d$ a ternary $d$--tuple let \[\mathbb{O}^{\geq 0}_s := \{x \in \mathbb{O}^{\geq 0} \,\,|\,\, x_i = ReLU(s_ix_i)\}\] denote the face of the non-negative orthant consisting of points whose $i$th component is $0$ when $s^i \leq 0$.
\end{definition}

The following lemma is immediate from the definitions.

\begin{lemma} \label{lem:ternlabelimage} Let $F$ be a ReLU neural network map of architecture $(n_0, \ldots, n_d)$, $\mathcal{C}(F)$ its canonical polyhedral complex, and $C$ a cell of $\mathcal{C}(F)$ with $\ell$th ternary label $s_C^{\ell}$. Then $F_{(\ell)}(C)$ is contained in $\mathbb{O}^{\geq 0}_{s_C^{\ell}}$.
\end{lemma}

\begin{definition} \label{defn:ternlabeldim} The {\em dimension} of a ternary label $s$, denoted $\mbox{dim}(s)$, is the dimension of the face $\mathbb{O}^{\geq 0}_{s}$. Equivalently, $\mbox{dim}(s)$ is the number of `$+1$'s in the tuple $s$.
\end{definition}

\section{Transversality} \label{ss:Transversality}

Recall the following classical notions (cf.~\citet{GuilleminPollack} and Section 4 of \citet{GrigsbyLindsey}):

\begin{definition} \citep{GuilleminPollack} \label{defn:maptransverse} Let $X$ be a smooth manifold with or without boundary, $Y$ and $Z$ smooth manifolds without boundary, $Z$ a smoothly embedded submanifold of $Y$, and $f:X \to Y$ a smooth map. We say that $f$ is {\em transverse} to $Z$ and write $f \pitchfork Z$ if 
\begin{equation} \label{eq:maptransverse}
df_p(T_pX) + T_{f(p)}Z = T_{f(p)}Y
\end{equation} 
for all $p \in f^{-1}(Z)$.
\end{definition}

\begin{definition}[Definition 4.2 of \citet{GrigsbyLindsey}] \label{defn:transoncells} Let $\mathcal{C}$ be a polyhedral complex in $\mathbb{R}^n$, let $f: |\mathcal{C}| \rightarrow \mathbb{R}^r$ be a map  which is smooth on all cells of $\mathcal{C}$, and let $Z$ be a  smoothly embedded submanifold (without boundary) of $\mathbb{R}^r$. We say that $f$ is  \emph{transverse on cells} to $Z$ and write $f \pitchfork_c Z$ if the restriction of $f$ to the {\em interior}, $\mbox{int}(C)$, of every cell $C$ of $\mathcal{C}$ is transverse to $Z$ (in the sense of Definition \ref{defn:maptransverse}).\end{definition}

Note that we use the convention that the interior of a $0$--cell is the $0$--cell itself.

The definition above can be extended so that $Z$ is the domain of a polyhedral complex in $\mathbb{R}^r$. This was essentially carried out in \citet{Masden}:

\begin{definition} \citep{Masden} \label{defn:twocpxstrans} Let $\mathcal{C}$ be a polyhedral complex in $\mathbb{R}^n$, let $f: |\mathcal{C}| \rightarrow \mathbb{R}^r$ be a map which is smooth on all cells of $\mathcal{C}$, and let $\mathcal{Z}$ be a polyhedral complex in $\mathbb{R}^r$. We say that $f$ and $\mathcal{Z}$ are \emph{transverse on cells} and write $f \pitchfork_c |\mathcal{Z}|$ if the restriction of $f$ to the interior of every cell of $\mathcal{C}$ is transverse to the interior of every cell of $|\mathcal{Z}|$ (in the sense of Definition \ref{defn:transoncells}).
\end{definition}

Transversality for intersections of polyhedral complexes implies that each non-empty cell in the intersection complex has the expected dimension. The following extension of the classical Map Transversality Theorem to polyhedral complexes is immediate (cf. \citet{GrigsbyLindsey} Cor. 4.7, \citet{Masden}):

\begin{corollary} Let $\mathcal{C}$ be a polyhedral complex in $\mathbb{R}^n$, let $f: |\mathcal{C}| \rightarrow \mathbb{R}^r$ be a map which is smooth on all cells of $\mathcal{C}$, and let $\mathcal{Z}$ be a polyhedral complex in $\mathbb{R}^r$ for which $f \pitchfork_c \mathcal{Z}$. Then for every pair of cells $C \in \mathcal{C}$ and $Z \in \mathcal{Z}$, $f^{-1}(Z) \cap \mbox{int}(C)$ is a (possibly empty) smoothly embedded submanifold of $\mbox{int}(C)$ whose codimension in $\mbox{int}(C)$ equals the codimension of $Z$ in $\mathbb{R}^r$. 
\end{corollary}

We use the standard convention that a manifold of negative dimension is empty. In particular, if $f: |\mathcal{C}| \rightarrow \mathbb{R}^r$ and $\mathcal{Z} \subseteq \mathbb{R}^r$ are transverse on cells as above, and $C \in \mathcal{C}$ is a $0$--cell, then $f^{-1}(Z) \cap C = \emptyset$ for all cells $Z$ of positive codimension.

\begin{lemma} \label{lem:polycpximagepolycpx} Let $\mathcal{C}$ be a polyhedral complex with domain $|\mathcal{C}| \subseteq \mathbb{R}^n$, and $F: |\mathcal{C}| \rightarrow \mathbb{R}^r$ a map that is affine-linear on cells. Then $F(\mathcal{C})$ is a polyhedral complex in $\mathbb{R}^r$.
\end{lemma}

Note that $F(\mathcal{C})$ need not be imbedded, nor even immersed.

\begin{proof} We begin by showing that the image, $F(C)$, of a $k$--dimensional cell (polyhedral set) $C \in \mathcal{C}$ is itself a polyhedral set in $\mathbb{R}^r$. By definition, $C$ is the solution set of finitely many affine-linear inequalities. That is, there exists (for some $m$) an $m \times n$ matrix $A$ and a vector $b \in \mathbb{R}^n$ such that 
\begin{equation}
C := \{x \in \mathbb{R}^n \,\,|\,\, Ax \geq b\}.
\end{equation} 
Let $V = \mbox{aff}(C)$ be the $k$--dimensional affine hull of $C$ and choose any point $p \in V$. Noting that $F$ is affine-linear on $V$, let $j$ denote the rank of $F$ restricted to $V$. 

Now choose an (affine) basis $\mathcal{B} = \{v_1, \ldots, v_k\}$ for $V$ whose final $(k-j)$ vectors form a basis for the affine kernel of the map $F$ restricted to $V$.  That is, all vectors $v$ of the form \[v = p + \sum_{i=j+1}^k a_iv_i\] satisfy $F(v) = F(p)$. 

Let $W = \mbox{Span}\{v_1, \ldots, v_j\}$. By construction, \[F|_V = F' \circ \pi_{V \rightarrow W}\] can be realized as the composition of the projection map $\pi_{V \rightarrow W}: V \rightarrow W$ and an affine-linear isomorphism $F': W \rightarrow F(V)$. It can be seen using the Fourier-Motzkin elimination method (cf. Sec. 12.2 in \cite{Schrijver}) that the image of $C$ under $\pi_{V \rightarrow W}$ is a polyhedral set, and it is immediate that the image of a polyhedral set under an affine-linear isomorphism is a polyhedral set. It follows that the image of $C$ under $F$ is a polyhedral set in $\mathbb{R}^r$. The continuity of $F$ ensures that if $C'$ is a face of $C$, then $F(C')$ will be a face of $F(C)$, so the image of $|\mathcal{C}|$ under $F$ will be the domain of a polyhedral complex, as desired.
\end{proof}

\begin{lemma} \label{lem:transpullback} Let $\mathcal{C}$ be a polyhedral complex in $\mathbb{R}^n$, $F: |\mathcal{C}| \rightarrow \mathbb{R}^r$ a map that is affine-linear on cells, and $\mathcal{Z}$ a polyhedral complex in $\mathbb{R}^r$. Let $F(\mathcal{C})$ denote the polyhedral complex in $\mathbb{R}^r$ that is the image of $\mathcal{C}$. If $i: |F(\mathcal{C})| \rightarrow \mathbb{R}^m$ is the inclusion map, then $i \pitchfork_c \mathcal{Z}$ iff $F \pitchfork_c \mathcal{Z}$. 
\end{lemma}

\begin{proof} Let $C$ be a cell of $\mathcal{C}$ with image $F(C) \in F(\mathcal{C})$, and let $Z \in \mathcal{Z}$. We will show that when $F$ (resp., $i$) is restricted to $\mbox{int}(C)$ (resp., to $\mbox{int}(F(C))$),  $F|_{\mbox{int}(C)} \pitchfork \mbox{int}(Z)$ iff $i|_{\mbox{int}(F(C))} \pitchfork \mbox{int}(Z)$. 

In the following, choose an affine-linear extension of $F$ to all of $\mathbb{R}^n$ and call it $F|^{\mathbb{R}^n}$. Note that any such extension can be decomposed as a projection onto the affine hull of $C$ followed by an affine isomorphism onto the affine hull of $F(C)$, as described in the proof of Lemma \ref{lem:polycpximagepolycpx}. Let $c := \mbox{dim}(C), c' = \mbox{dim}(F(C)),$ and $z = \mbox{dim}(F^{-1}(Z)), z' = \mbox{dim}(Z)$. Begin by noting that $i|_{\mbox{int}(F(C))} \pitchfork \mbox{int}(Z)$ iff $F(C) \cap Z$ is a polyhedral set of dimension $(c'+z) - r$ iff either they have empty intersection or the interiors of $F(C)$ and $Z$ have non-empty intersection and the affine hulls of $F(C)$ and $Z$ intersect in an affine-linear space of codimension $(r-c')+(r-z')$ (that is, of dimension $(c'+z') - r$).

The statement that $F|_C \pitchfork Z$ iff $i|_{F(C)} \pitchfork Z$ is vacuous in the empty intersection case. 

In the non-empty intersection case, the affine-linear map $F$ restricted to $\mbox{int}(C)$ is either a homeomorphism onto its image or a linear projection map onto a set homeomorphic to its image, $\mbox{int}(C) \cap \mbox{int}(Z) \neq \emptyset$ iff $\mbox{int}(C) \cap \mbox{int}(F^{-1}(Z))$. Moreover, the rank-nullity theorem applied to $F|^{\mathbb{R}^n}$ tells us that $\mbox{aff}(F(C)) \cap \mbox{aff}(Z)$ has dimension $(c'+z')-r$ iff $\mbox{aff}(C) \cap F^{-1}(\mbox{aff}(Z))$ has dimension $(c+z) - n$.

It follows that for all cells $C$ of $\mathcal{C}$ and $Z$ of $\mathcal{Z}$, $F|_{\mbox{int}(C)} \pitchfork Z$ iff $i|_{F(C)} \pitchfork Z$, and the statement follows.
\end{proof}

\begin{definition}[Definition 11 of \citet{Masden}]  \label{defn:supertransversal}
Let $F$ be a ReLU neural network of depth $d$, and let the $(i-1)$st layer map, $F^{i-1}: \mathbb{R}^{n_{i-2}} \rightarrow \mathbb{R}^{n_{i-1}}$ (resp., the composition of maps from the $i$th layer map, $F^{(i)}: \mathbb{R}^{n_{i-1}} \rightarrow \mathbb{R}^{n_d}$) be viewed as maps that are affine-linear on cells of their respective canonical polyhedral decompositions, $\mathcal{C}(F^{i-1})$ (resp. $\mathcal{C}(F^{(i)})$).  If, for all $2 \leq i \leq d$, we have \[F^{i-1} \pitchfork_c \mathcal{C}(F^{(i)}),\] then we call F a {\em supertransversal} neural network.
\end{definition}

Informally, we can think of supertransversality as the right generalization of the genericity condition for hyperplane arrangements to bent hyperplane arrangements. Recall that a hyperplane arrangement is {\em generic} if every $k$--fold intersection of hyperplanes in the arrangement is an affine linear subspace of dimension $n-k$.  Analogously, it follows from the definitions that every $k$--fold intersection of bent hyperplanes associated to a generic, supertransversal network intersect in a (possibly empty) polyhedral complex of dimension $n-k$. 

An important result proved in \citet{Masden} (see also Theorem 3 of \citet{GrigsbyLindsey}) is the following:

\begin{proposition}[Lemma 12 of \citet{Masden}] For any neural network architecture, the set of parameters associated to generic, supertransversal marked neural network functions is full measure in parameter space, $\mathbb{R}^D$.
\end{proposition}

\begin{definition} \label{def:TPIC}
Let $s \in (\Omega = \mathbb{R}^D)$ be a generic, supertransversal (Definition \ref{defn:supertransversal}) parameter for a ReLU neural network of architecture $(n_0, \ldots, n_d)$. We say that $s$ satisfies the {\em transverse pairwise intersection condition (TPIC)} for all adjacent layer maps if $\hat{H}^{\ell}_i \pitchfork_c \hat{H}^{\ell + 1}_j \neq \emptyset$  for all $i,j, \ell.$ That is, every pair of bent hyperplanes in adjacent layers has non-empty transverse intersection.
\end{definition}

In the language of \citet{RolnickKording} and \citet{PhuongLampert}, a generic, supertransversal parameter $s$ satisfies the transverse pairwise intersection condition (TPIC)  for all adjacent layer maps iff every pair of nodes in every pair of adjacent layers of the {\em dependency graph} is connected by an edge.

\section{Unbounded Solyhedral Sets and Sufficiently High-Bias Positive-Axis Hyperplanes}

In order to choose parameters whose bent hyperplane arrangement satisfies (TPIC), we will need to establish some results about the images of unbounded polyhedral sets under generic, supertransversal ReLU neural network layer maps. We will also need to understand the intersections of these images with sufficiently high-bias positive-axis hyperplanes. 

The following proposition ensures that the images of the nested unbounded polyhedral sets $\mathcal{S}_1 \subseteq \mathcal{S}_2 \subseteq \ldots \subseteq \mathcal{S}_d$ referenced in the proof of the main theorem are unbounded in the layers of the neural network.

\begin{proposition} \label{prop:unbddimage} Let $F_\theta: \mathbb{R}^{n_0} \rightarrow \mathbb{R}^{n_d}$ be a generic, supertransversal ReLU network map of architecture $(n_0=k, n_1, \ldots, n_d)$ with $n_\ell \geq k$ for all $\ell$, and let $\mathcal{S}$ be an unbounded polyhedral set of dimension $k$ in the canonical polyhedral complex $\mathcal{C}(F_\theta)$. If the sign sequence $s_S = (s^1, \ldots, s^{n_d})$ associated to $\mathcal{S}$ satisfies $s^\ell = (\underbrace{+1, \ldots, +1}_{k}, \underbrace{-1, \ldots -1}_{n_\ell - k})$  for all $i \leq \ell$, then $F_{(\ell)}(\mathcal{S})$ is an unbounded polyhedral set of dimension $k$ contained in $\mathbb{O}^{+}_k \subseteq \mathbb{R}^{n_\ell}$.
\end{proposition}

\begin{proof} The fact that $F_{(\ell)}(S)$ is a polyhedral set contained in $\mathbb{O}^{\geq 0}_k \subseteq \mathbb{R}^{n_{\ell}}$ is a consequence of Lemmas \ref{lem:ternlabelimage} and \ref{lem:polycpximagepolycpx}, so we need only prove that its image is unbounded, of dimension $k$. 

We will prove this by induction on $\ell$. When $\ell = 1$, $\mathcal{C}(F_\theta) = \mathcal{C}(A^1)$ for the generic hyperplane arrangement $\mathcal{A}^1$ associated to $F^1$. Since $\mathcal{S}$ is a {\em pointed} (since $n_1 \geq k$) unbounded polyhedral set of dimension $k$, its boundary contains unbounded $1$--cells (rays) $R_i = \{x_i + tv_i \,\, | \,\, t \geq 0\}$ based at $\{x_i\}$ with {\em slopes} $\{v_i\}$. Note that because $\mathcal{A}^1$ is generic, each $0$--cell $x_i$ is a $k$--fold intersection of distinct hyperplanes from $\mathcal{A}^1$, and each $R_i$ is contained in a $(k-1)$--fold intersection of distinct hyperplanes from $\mathcal{A}^1$. It follows that $\mathcal{S}$ has at least $k$ unbounded facets, each contained in a different hyperplane of $\mathcal{A}^1$. Reindex if necessary so $k$ of these are $H_1^1, \ldots, H_k^1$ and then flip co-orientations so that the sign sequence on $\mathcal{S}$ is $$s^1_{\mathcal{S}} = (\underbrace{+1, \ldots, +1}_{k}, \underbrace{-1, \ldots -1}_{n_1 - k}).$$

We have chosen the first $k$ neurons of $F^1$ to be ``on" on (the interior of) $\mathcal{S}$. This implies that if we let $w_1, \ldots, w_{n_1}$ be the weight vectors and $b_1, \ldots, b_{n_1}$ the biases associated to $F^1$, we have $w_i \cdot x_j +b_i\geq 0$ for all $1 \leq i \leq k$, with $w_i \cdot x_j +b_i = 0$ iff $x_j \in H_i$.

As to the slopes $\{v_i\}$ of the unbounded $1$--cells (rays) $\{R_i\}$, it is immediate that each $\{v_i\} \in \mbox{Cone}(S)$ (cf.~\citet{Schrijver}), 
and by reindexing if necessary we may assume $\{v_{1}, \ldots, v_{k}\}$ is a basis for $\mathbb{R}^{n_0 = k}$ since $\mathcal{S}$ has dimension $k$. Moreover, we claim that $w_i \cdot v_j \geq 0$ for all $1 \leq i, j \leq k$ with $w_i \cdot v_j = 0$ iff $R_j \subseteq H_i$. 

To see this claim, note first that if $w_i \cdot x_j + b_i> 0$, then $x_j \not\in H_i$, so $R_i \not\subset H_i$. We also see that $w_i \cdot v_j \geq 0$, since otherwise there would be some $t > 0$ for which $w_i \cdot (x_j + tv_j) = 0$, contradicting the unboundedness of $R_i$. But $w_i \cdot v_j \neq 0$ since this would imply $v_j \in H_{i_1} \cap \ldots \cap H_{i_{k-1}} \cap w_i^\perp$, which contradicts the genericity of $\mathcal{A}^1$. 

If $w_i\cdot x_j + b_i= 0$, then $x_j \in H_i$ and hence $R_j \subset H_i$ iff $w_i \cdot v_j = 0$, as desired.
 
Now let $$W = \left[\begin{array}{c} w_1^T\\ \vdots \\ w_{n_{\ell}}^T\end{array}\right]$$ be the matrix whose row vectors are the weight vectors 
 $\{w_i\}$ associated to $\mathcal{A}^1$ and let $$V := \left[\begin{array}{ccc}v_1 & \cdots & v_k\end{array}\right].$$ Then the $i$th column of $WV$ is precisely the pre-activation image of the vector $v_i$ in $\mathbb{R}^{n_1}$. Moreover, since $\mathcal{A}$ is generic, each $k\times k$ minor of $WV$ has rank $k$. We also just saw above that the initial $k \times k$ minor of $WV$ is unaffected by the ReLU activation, since all entries are $\geq 0$. 
 
 Since the post-activation rank of $WV$ is the dimension of $\mbox{Cone}(F^1(\mathcal{S}))$, we conclude that $F^1(\mathcal{S})$ is unbounded, of dimension $k$, as desired, and the base case is complete.
 
Now suppose $\mathcal{S}$ satisfies the assumptions and we know that $F_{(\ell -1)}(\mathcal{S}) \subseteq \mathbb{O}^{\geq 0}_k \subseteq \mathbb{R}^{n_\ell -1}$ is unbounded of dimension $k$. 

As in the proof of the base case, consider the unbounded rays $R_i$ of $F_{(\ell-1)}(\mathcal{S})$, their basepoints $x_i$, and their slopes $\{v_i\} \in \mbox{Cone}\left(F_{(\ell-1)}(\mathcal{S})\right)$. As before, assume that $v_1, \ldots, v_k$ gives a basis for the initial $k$--plane of $\mathbb{R}^{n_{\ell-1}}$ and let $w_1, \ldots, w_{n_\ell}$ be the weight vectors of $F^\ell$, and let $W$ be the matrix whose rows are $w_i^T$ and $V$ the matrix whose columns are $v_j$. By exactly the same argument as before, we see that the initial $k \times k$ minor of $WV$ is unaffected by the ReLU activation, and since $F^\ell$ is generic and super-transversal to all previous layer maps, each $k \times k$ minor of $WV$ is rank $k$, which--as before--implies that $F_{(\ell)}(\mathcal{S}) = F^\ell(F_{(\ell-1)}(\mathcal{S}))$ is also unbounded, of dimension $k$, as desired.
 \color{black}

\end{proof}

\begin{corollary} \label{cor:AffIsomorph} Let $F_\theta: \mathbb{R}^{n_0} \rightarrow \mathbb{R}^{n_d}$ be a generic, supertransversal ReLU network map of architecture $(n_0=k, n_1, \ldots, n_d)$ with $n_\ell \geq k$ for all $\ell$, and let $\mathcal{S}$ be a non-empty unbounded polyhedral set of dimension $k$ in the canonical polyhedral complex $\mathcal{C}(F_\theta)$. If the sign sequence $s_S = (s^1, \ldots, s^{n_d})$ associated to $\mathcal{S}$ satisfies 
$s^\ell = (\underbrace{+1, \ldots, +1}_{k}, \underbrace{-1, \ldots -1}_{n_\ell - k})$  for all $i \leq \ell$, then $F_{(\ell)}$ restricted to the (non-empty) interior of $\mathcal{S}$ is a rank $k$ affine-linear map, hence a homeomorphism onto its image for all $i \leq \ell$.
\end{corollary}

\begin{proposition} \label{prop:perturb} Let $\mathcal{S}$ be a non-empty unbounded polyhedral set of dimension $k$ in $\mathbb{O}^{\geq 0}_k \subseteq \mathbb{R}^{n_{\ell-1}}$, viewed as the domain of a polyhedral complex that also contains all of its faces. Suppose $H \subseteq \mathbb{R}^{n_{\ell - 1}}$ is a hyperplane for which $H \pitchfork_c \mathcal{S} \neq \emptyset$. For any $n_\ell \geq k$, we can always find $n_\ell$ hyperplanes $H_1, \ldots, H_{n_\ell}$ satisfying:
\begin{enumerate}
	\item $H_i \pitchfork_c \mathcal{S} \neq \emptyset$,
	\item the restricted hyperplane arrangement $\mathcal{K} = \{K_1, \ldots K_{n_\ell}\}$ is generic, and 
	\item the bounded subcomplex of the restricted hyperplane arrangement $\{K_1, \ldots K_{n_\ell}\}$ is non-empty and contained in the interior of $\mathcal{S}$. 
\end{enumerate} 
\end{proposition}

\begin{proof} 

Let $H_1 := H$. Choose a point $p \in H \cap \mathcal{S}$ in the (non-empty) interior of $\mathcal{S}$, and a small open neighborhood $N(p) \subseteq \mathcal{S}$. We can pick slight (non-generic) perturbations $H_2, \ldots, H_{n_\ell}$ of $H$ so that $p \in H_i$ for all $i$. Since transverse intersection is an open condition, we can by further perturbation insure that  $H_2, \ldots, H_{n_\ell}$ still intersect $\mathcal{S}$ transversely. Because generic hyperplane arrangements are dense and open in parameter space, we can by further perturbation insure that the restricted hyperplane arrangement $\{K_1, \ldots, K_{n_\ell}\}$ is generic in the initial coordinate $k$--plane $\mathbb{R}^{n_\ell}_k$ and that the bounded subcomplex of this arrangement is contained in $N(p)$, as desired.
\end{proof}

\begin{lemma} \label{lem:intunbddcell} Let $F: \mathbb{R}^{n_{\ell -1}} \rightarrow \mathbb{R}^{n_\ell}$ be a generic ReLU neural network layer map with associated co-oriented generic hyperplane arrangement $\mathcal{A}$, and let $\mathcal{C}(F) = \mathcal{C}(\mathcal{A})$ be the associated polyhedral decomposition of $\mathbb{R}^n$. For almost all {\em positive} weight vectors $\vec{w} \in \mathbb{R}^{n_\ell}$ there exists a negative bias $b \in \mathbb{R}$ such that the corresponding positive-axis hyperplane: \[H := \{\vec{x} \in \mathbb{R}^{n_\ell} \,\,|\,\, \vec{w}\cdot \vec{x} + b = 0\}.\] has non-empty transverse intersection with $F(C)$ for each unbounded $n_{\ell-1}$--cell $C$ in $\mathcal{C}(\mathcal{A})$ whose ternary labeling has dimension $\geq 1$.
\end{lemma}

\begin{proof} 
By an argument analogous to the one in the proof of Proposition \ref{prop:unbddimage}, we see that the image of any unbounded polyhedral set $C$ in $\mathcal{C}(\mathcal{A})$ of dimension $d\geq 1$ whose associated sign sequence $s_C$ has dimension $\mbox{dim}(s_C)$ (Definition \ref{defn:ternlabeldim}) is an unbounded polyhedral set of dimension $\mbox{min}(d, \mbox{dim}(s_C))$, hence contains an unbounded ray in the non-negative orthant $\mathbb{O}^{\geq 0}$. But for every ray $R$ contained in the non-negative orthant and every positive weight vector there exists a sufficiently high negative bias such that the corresponding sufficiently-high bias positive-axis hyperplane intersects $R$. Since transverse intersection is an open condition, we can perturb $w$ slightly so this intersection is transverse.
\end{proof}

\begin{lemma} \label{lem:TransByCvxCone} Let $B \subseteq \mathbb{R}^n$ be a bounded set, and let $S \subseteq \mathbb{R}^n$ be any unbounded pointed polyhedral set of dimension $n$. There exists $\vec{v} \in \mbox{Cone}(S)$ such that the translate of $B$ by $\vec{v}$ is in $S$. That is, there exists $\vec{v} \in \mbox{Cone}(S)$ such that $b+ \vec{v} \in S$ for all $b \in B$.
\end{lemma}

\begin{proof} Let $P$ be any polytope containing $B$ and let $V$ be its (finite) set of $0$--faces, so $P$ is the convex hull of $V$. 


It suffices to prove that when $\mbox{Cone}(S)$ is non-empty and has dimension $n$ (as is the case for any unbounded polyhedral set $S$ of dimension $n$), the set $S + -\mbox{Cone}(S)$ is all of $\mathbb{R}^n$. This will imply that $P$, being the convex hull of finitely many points, can be realized as $P' + w$ for $P'$ the convex hull of finitely many points in $S$ and $w \in -\mbox{Cone}(S)$.

But the fact that $\mathbb{R}^n = S + -\mbox{Cone}(S)$ follows immediately from the fact that there exist vectors $v_1, \ldots, v_n \in \mbox{Cone}(S)$ that form a basis for $\mathbb{R}^n$.
\color{black}
\end{proof}

The following result is well-known (cf. \citet{Stanley}), but we include a proof here for completeness.

\begin{lemma} \label{lem:genericfullmeasure} Let $\mathbb{R}^D$ be the parameter space for a feedforward ReLU network architecture. There exists an algebraic set $S$ of positive codimension (and Lebesgue measure $0$) such that every parameter $\theta \in \mathbb{R}^D \setminus S$ is generic.
\end{lemma}
\begin{proof}
A parameter is generic iff each hyperplane arrangement associated to each layer map is generic. Since there are finitely many layer maps, and the union of finitely many algebraic sets is an algebraic set, we need only prove that the statement of the lemma holds for a parameter associated to a single layer map.

Classical results in linear algebra relating the solution sets of homogeneous and inhomogeneous linear systems tell us that an arrangement of $n_\ell$ affine hyperplanes in $\mathbb{R}^{n_{\ell-1}}$ is generic iff the associated bias-free (central) hyperplane arrangement has the property that every $k$--fold intersection of (non-affine) hyperplanes intersects in a linear subspace of codimension $k$ (where a linear subspace of codimension $>n_\ell$ is by definition empty).

The rank-nullity theorem tells us that this happens iff every weight matrix associated to a $k$--fold subset of the central arrangement has full rank, $\min\{k, n_{\ell-1}\}$. Noting that the rank of a $k \times n_{\ell - 1}$ matrix is the maximum $m$ for which some $m \times m$ minor has nonzero determinant, and the determinant is a polynomial equation in the matrix entries, we conclude that away from an algebraic set the parameters are generic. Since a non-empty algebraic set always has positive codimension (and hence Lebesgue measure $0$), the conclusion follows.
\end{proof}

\begin{corollary} \label{cor:genericopen} Let $\theta_0 \in \mathbb{R}^D$ be a generic parameter. There exists an open neighborhood $U$ of $\theta_0$ such that every parameter $\theta \in U$ is generic.
\end{corollary}

\begin{proof}
An algebraic set is closed, hence its complement is open.
\end{proof}
\color{black}

\section{Linear Regions Assumption (LRA) and Transverse Parwise Intersection Condition (TPIC)}\label{sec:LRA}

\begin{definition}
A {\em linear region} of a continuous, finitely piecewise linear function $f:\mathbb{R}^{n_0} \to \mathbb{R}^{n_d}$ is any maximal connected set $S \subset \mathbb{R}^{n_0}$ such that the restriction of $f$ to $S$ is affine-linear.
\end{definition} 

Note that each linear region is a closed set.

\begin{definition} \label{def:LRA}
 Let $T \subseteq \mathbb{R}^{n_0}$ be a union of the closures of $\pm$ activation regions for a   ReLU network map $F_\theta: \mathbb{R}^{n_0} \rightarrow \mathbb{R}^{n_d}$. 
 $F_\theta$  is said to satisfy the {\em Linear Regions Assumption} (LRA)  on $T$  if each linear region of the restriction of $F_\theta$ to $T$  is its intersection with the closure of a single $\pm$-activation region of $F_\theta$.
\end{definition}


The following Lemma is an immediate consequence of the definition of LRA. (Compare Remark \ref{rem:GenSupTransCase1skeleton}.)

\begin{lemma} \label{lem:1skeletonforCombStabLRA}
For a parameter that satisfies LRA  on $T \subseteq \mathbb{R}^n_0$ as above, the  intersection of $T$ with the \color{black} union of the bent hyperplanes coincides with the domain of the $(n_0-1)$-skeleton of the canonical polyhedral complex. 
\end{lemma} 

In order to deduce that there are no hidden symmetries on a {\em positive measure} subset of parameter space, it will be important for us to know that the required LRA condition is satisfied not just for the construction we give in Section \ref{sec:GeomComb} but also on a full measure subset of an open neighborhood of the associated parameter. We turn to establishing the foundations for proving this now.

Recall the following definition and result (cf. \citet{HaninRolnick}, \citet{GLMW}), which tell us that for most pairs $(\theta, x) \in \Omega \times \mathbb{R}^{n_0}$, each coordinate of the realized function $F_\theta(x)$ is expressible as a polynomial in the coordinates of the parameter and the input.

\begin{definition} \label{defn:paramcoordrealize} 
Let $\Omega$ be the parameter space for ReLU network architecture $(n_0, \ldots, n_d)$. Denote by $\mathcal{F}: \Omega \times \mathbb{R}^{n_0} \rightarrow \mathbb{R}^{n_d}$ the function $\mathcal{F}(\theta,x) := F_\theta(x).$
\end{definition}

\begin{lemma}[\citet{GLMW}, Lemma 3.4]\label{l:3.4}
Let $\Omega$ be the parameter space for ReLU network architecture $(n_0,\ldots,n_d)$. Let $\theta \in \Omega$, and suppose $x \in \mathbb{R}^{n_0}$ is in a $\pm$--activation region for $F_\theta$. 
Then there is an open neighborhood of $(\theta,x) \in \Omega \times \mathbb{R}^{n_0}$ on which each coordinate of $\mathcal{F}$ is a polynomial in the coordinates of $\theta$ and $x$.
\end{lemma}

Indeed, if $\theta$ is moreover a generic and supertransversal parameter and $x \in \mathbb{R}^{n_0}$ is any point in the domain, it is well-known that we can calculate this polynomial explicitly in terms of the ternary labeling on $x$ and the parameters of $\theta$, cf. Lemma 8 of  \citet{HaninRolnick}. For completeness, we describe this process here. We first establish some notation, then describe how the polynomial is calculated in Lemma \ref{lem:poly}.

\begin{definition} \label{defn:weightedcomputationalgraph} The {\em augmented} computational graph $\tilde{G}$ for the feedforward ReLU network architecture $(n_0, \ldots, n_d)$ is the graded oriented graph: 
\begin{itemize}
    \item with $n_\ell$ ordinary vertices and $1$ distinguished vertex of grading $\ell$ for $\ell = 0, \ldots, d-1$, and $n_d$ ordinary vertices of grading $d$,
    \item for every $\ell = 0, \ldots, d-1$, every vertex of grading $\ell$ is connected by a single oriented edge to every ordinary vertex of grading $\ell+1$, oriented toward the vertex of grading $\ell+1$.
\end{itemize} 
\end{definition}

One obtains the augmented computational graph for an architecture from the standard computational graph for the architecture by adding an extra marked vertex for each non-output layer, whose purpose is to record the bias term in each affine-linear map. Accordingly, given a parameter $\theta$ one obtains a labeling of the edges of the augmented computational graph:
\begin{itemize}
    \item the edge from the distinguished vertex of layer $\ell$ to the $k$th ordinary vertex of layer $\ell+1$ is labeled with $b^{\ell+1}_k$, the $k$th  component of the bias vector for $F^{\ell+1}$,
    \item the edge from the $i$th ordinary vertex of layer $\ell-1$ to the $j$th ordinary vertex of layer $\ell$ is labeled with $W^{\ell}_{ji}$.
\end{itemize}

Associated to every oriented path $\gamma$ is a corresponding monomial, $m(\gamma)$, in the parameters obtained by taking the product of the parameters on the edges traversed along $\gamma$. See Figure \ref{fig:AugmentedCompGraph}.

\begin{SCfigure}
\includegraphics[width=3in]{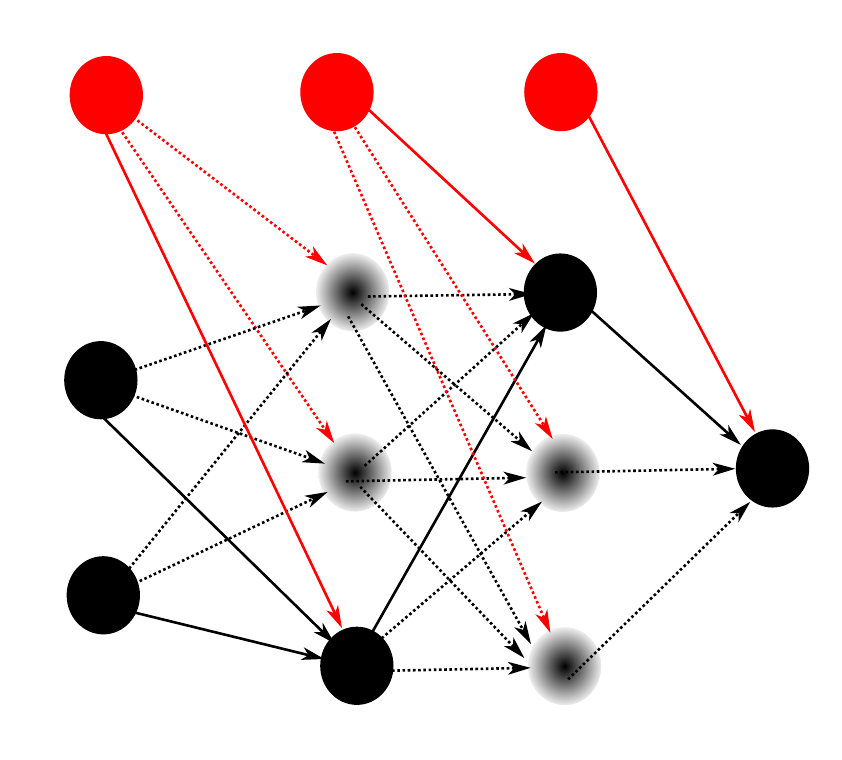}
  \caption{An augmented computational graph for architecture (2,3,3,1). The ordinary vertices are black, and the distinguished vertices are red. Black edges are labeled with weights, and red edges are labeled with biases. A complete path is one that ends at an output vertex and begins either at an input vertex or at one of the distinguished vertices. In the diagram above, we have blurred out vertices corresponding to inactive neurons associated to an input vector $x$ with ternary label $s_x = (s_x^1,s_x^2,s_x^3) = ((-1,0,+1), (+1,0,0), (+1))$. The open paths associated to this ternary label are the ones in the diagram above with solid (non-dashed) edges. The reader can check that there are three open, complete paths $\gamma, \gamma', \gamma'' \in \Gamma_{x,*}$,  whose  monomials are $m(\gamma) = b_3^1W^2_{13}W^3_{11}$, $m(\gamma') = b_1^2W^3_{11}$, and $m(\gamma'') = b_1^3$. There is a unique open, complete path $\gamma_1 \in \Gamma_1$, with monomial $m(\gamma_1) = W^1_{31}W^2_{13}W^3_{11}$ and a unique open, complete path $\gamma_2 \in \Gamma_2$, with monomial $m(\gamma_2) = W^1_{32}W^2_{13}W^3_{11}$.}
  \label{fig:AugmentedCompGraph}
\end{SCfigure}

\begin{definition} 
\label{defn:openpath} Let $\theta \in \Omega$ be a generic, supertransversal parameter, and let $x \in \mathbb{R}^{n_0}$ be any point in the domain, with associated ternary labeling $s_x = (s^1_x, \ldots, s^d_x)$. A path $\gamma$ is said to be \emph{open at $x$ for parameter $\theta$} if every node along $\gamma$ has ternary labeling $+1$.
\end{definition}

\begin{definition}
\label{defn:completepath}
Let $\tilde{G}$ be the augmented computational graph for the ReLU network architecture $(n_0, \ldots, n_d)$. A path $\gamma$ is said to be {\em complete} if it ends at a vertex in the output layer and begins at either a vertex of the input layer or at one of the distinguished vertices in a non-input layer. For $\theta \in \Omega$ and  each $1 \leq k \leq n_d$ we will denote by  
\begin{itemize} 
    \item $\Gamma_{x}^{\theta,k}$ the set of complete paths that are open at $x$ for the parameter $\theta$  and end at the $k$th node of the output layer, 
    \item by $\Gamma_{x,i}^{\theta,k} \subseteq \Gamma_x^{\theta,k}$ the subset of $\Gamma_x^{\theta,k}$ beginning at input node $i$, and 
    \item by $\Gamma_{x,*}^{\theta,k}\subseteq \Gamma_x^{\theta,k}$ the subset of $\Gamma_x^{\theta,k}$ beginning at one of the distinguished vertices.
\end{itemize}  
\end{definition}

The following lemma is well-known to the experts (e.g. Lemma 8 of \citet{HaninRolnick}).

\begin{lemma} \label{lem:poly}
Let $\theta \in \Omega$ be a generic, supertransversal parameter, and let $x = (x_1, \ldots, x_{n_0}) \in \mathbb{R}^{n_0}$ be any point in the domain. For $k=1, \ldots, n_d$, the polynomial associated to the $k$th output component of $F_\theta$ at $x$ is given by 
\[\mathcal{F}_k(\theta,x) =  \sum_{\gamma \in \Gamma_{x,*}^{\theta,k}}  m(\gamma) + \sum_{i=1}^{n_0} x_i\sum_{\gamma \in \Gamma_{x,i}^{\theta,k}} m(\gamma) .\]   
\end{lemma}

\begin{remark} \label{rem:onlysignseqmatters}
Fix $\theta$ and any ternary labeling $s$.  Define
\[
(\mathbb{R}^{n_0}_s \times \Omega)_s \coloneqq \{(x,\theta) \in \mathbb{R}^{n_0} \times \Omega \mid \textrm{ the ternary labeling of } x \textrm{ with respect to } \theta \textrm{ is }s \}.\]
Then $\mathcal{F}$ restricted to $(\mathbb{R}^{n_0}_s \times \Omega)_s $ is a  vector of polynomial functions. In the proof of the Lemma below, we will use the notation $\mathcal{F}_s$ for this vector of polynomial functions. 
\end{remark}

In the Lemma below, recall from Section \ref{sec:GeomComb} that $\hat{H}_i^{\ell}$ denotes the bent hyperplane in the domain $\mathbb{R}^{n_0}$ associated to the $i$th neuron of the $\ell$th layer map.

\begin{lemma} \label{lem:openpaths} 
There exists an algebraic, measure 0 set $B \subset \Omega$ such that for any generic, supertransversal parameter $\theta \in \Omega \setminus B$ 
and any point  $x \in \hat{H}_i^{\ell - 1} \cap \hat{H}_j^{\ell}$ (for parameter $\theta$), 
 if $\Gamma_x\setminus \Gamma_{x,*}$ is non-empty, then the LRA is satisfied on the union, $T_{ij}^{\ell}$, of the closures of the four $\pm$-activation regions adjacent to $x=p_{ij}^{\ell}$.
\end{lemma}

\begin{proof}

Fix $i,j$. Fix a ternary labeling $s \in \{-1,0,1\}^{n_1} \times \ldots \times \{-1,0,1\}^{n_d} $.  Suppose there exists a point $x \in \hat{H}_i^{\ell - 1} \cap \hat{H}_j^{\ell}$ with ternary labeling $s=s_x$.    

By supertransversality, every non-empty bent hyperplane besides $\hat{H}_{1}^{\ell-1}$ and $\hat{H}_1^{\ell}$ has positive minimal distance to $x$. It follows that a sufficiently small neighborhood of $x$ contains points only in the closures of the four $\pm$-activation regions adjacent to $x$.
\color{black}
    By permuting neurons in layers $\ell-1$ and $\ell$ if necessary, we may assume without loss of generality that $i=j=1$. It follows that the first component of each of the ternary labels $s_x^{\ell-1}$ and $s_x^{\ell}$ is $0$. Accordingly, the ternary labelings of the four $\pm$-activation regions adjacent to $x$ agree with those of $x$ except at the first coordinates of $s^{\ell-1}, s^\ell$ . It is therefore natural to label the adjacent $\pm$-activation regions by $++$, $+-$, $-+$, $--$ according to whether the $1$st coordinate of $s^{\ell-1}, s^\ell$ is $\pm 1$. Let $x_{++}, x_{+-}, x_{-+}, x_{--}$ be points in the corresponding $\pm$-activation regions adjacent to $x$, and let $s_{++}, s_{+-}, s_{-+}, s_{--}$ be the corresponding ternary labelings.

    The assumption that $\Gamma_x \setminus \Gamma_{x,*}$ is non-empty tells us that there is at least one open, complete path at $x$, which implies that each of the ternary labelings $s_x^1, \ldots, s_x^d$ has at least one $+1$ component. In other words, there is at least one neuron active in each layer at the input $x = p^\ell_{ij}$.

Let $\mathcal{F}_{s_{++}}, \mathcal{F}_{s_{+-}}, \mathcal{F}_{s_{-+}}, \mathcal{F}_{s_{--}}$ be the  vector of  polynomials as in Remark \ref{rem:onlysignseqmatters}. 
    Lemma \ref{lem:poly} tells us how to compute these four polynomials. 
    We will show that the  vectors of polynomials $\mathcal{F}_{s_{++}}, \mathcal{F}_{s_{+-}}, \mathcal{F}_{s_{-+}}, \mathcal{F}_{s_{--}}$
 are pairwise distinct. I.e.,  viewing the summands of the polynomial components as consisting of a coefficient that is an algebraic expression in $\theta$ and a variable $x_i$, we will show that different polynomials have different coefficients.  
   Then we let $B_{i,j,s}$ denote the set of parameters $\theta$ such that two or more of the restrictions $\mathcal{F}_{s_{++}}(\theta, \cdot), \mathcal{F}_{s_{+-}}(\theta,\cdot), \mathcal{F}_{s_{-+}}(\theta,\cdot), \mathcal{F}_{s_{--}}(\theta, \cdot)$ coincide. The set $B_{i,j,s}$ is an algebraic set (a finite union of solutions to polynomials) and hence has measure $0$. It follows that the set $B \coloneqq \bigcup_{i,j,s} B_{i,j,s}$ is an algebraic set of $0$ measure that has the desired properties.

Thus, we turn to proving that $\mathcal{F}_{s_{++}}, \mathcal{F}_{s_{+-}}, \mathcal{F}_{s_{-+}}, \mathcal{F}_{s_{--}}$ are pairwise distinct polynomials. Explicitly, let $v_1^{\ell-1}$ (resp., $v_1^\ell$) be the first ordinary vertex in layer $\ell -1$ (resp., in layer $\ell$). Consider the set of paths in $\Gamma_{x_{++}}$ passing through an ordinary vertex from layer $\ell - 1$ and an ordinary vertex from layer $\ell$. It is immediate that this set can be decomposed into the disjoint union of:
    \begin{itemize}
        \item the set of paths through both $v_1^{\ell-1}$ and $v_1^\ell$, which we will denote by $\Gamma_{11}$
        \item the set of paths through $v_1^{\ell-1}$ and not $v_1^{\ell}$, which we will denote by $\Gamma_{1*}$
        \item the set of paths through $v_1^{\ell}$ and not $v_1^{\ell-1}$, which we will denote by $\Gamma_{*1}$
        \item the set of paths through neither $v_1^{\ell-1}$ nor $v_1^{\ell}$, which we will denote by $\Gamma_{**}$
    \end{itemize}

    Now note that $\Gamma_{x_{--}} = \Gamma_{x}$. Since $\Gamma_x$ is non-empty each of $\Gamma_{x_{-+}}, \Gamma_{x_{+-}}, \Gamma_{x_{++}}$ is non-empty as well, which implies that the polynomial in $(\theta,x)$ associated to each of these sets is nonzero. 

    Next, note that $\Gamma_{x_{+-}}= \Gamma_{x_{--}} \cup \Gamma_{1*}$. Since there is at least one neuron in each layer active at $x_{+-}$, the set $\Gamma_{1*}$ is non-empty, and hence the polynomial \[\sum_{\gamma \in \Gamma_{1*}} m(\gamma)\] is nonzero. This tells us that the polynomials $\mathcal{F}_{s_{--}}$ and $\mathcal{F}_{s_{+-}}$ associated to $\Gamma_{x_{--}}$ and $\Gamma_{x_{+-}}$ are distinct (as functions of two variables $\theta$ and $x$, and affine-linear in $x$). 
    Similarly, the polynomial associated to $\Gamma_{x_{-+}}$ is distinct from that associated to $\Gamma_{x_{--}}$. 

    Indeed, the fact that each of $\Gamma_{11}, \Gamma_{1*}, \Gamma_{*1},$ and $\Gamma_{**}$ contains a path (and hence a monomial containing a distinct weight) not present in the others implies, by an analogous argument, that the polynomials associated to $\Gamma_{x_{++}}, \Gamma_{x_{+-}}, \Gamma_{x_{-+}}, \Gamma_{x_{--}}$ are all pairwise distinct.
\color{black}
    
\end{proof}

\color{black}

In \citet{Masden}, it is proved that for a generic, supertransversal parameter $\theta$, the map \[s: \mathcal{C}(F_\theta) \rightarrow \{-1,0,+1\}^{n_1 + \ldots + n_d}\] that assigns to each cell $C$ of the polyhedral complex, $\mathcal{C}(F_\theta)$, its ternary activation pattern, $s_C$, is well-defined, injective, and has the property that $C$ is a $k$-cell of $\mathcal{C}(F)$ if and only if $s(C)$ has exactly $n_0 - k$ entries which are $0$.

That is, there is no ambiguity in defining a ternary activation pattern for a cell $C$ of $\mathcal{C}(F_\theta)$, each possible ternary activation pattern $s \in \{-1,0,+1\}^{n_1 + \ldots + n_d}$ is in the image of {\em at most} one polyhedral set $C$ in $\mathcal{C}(F_\theta)$, and the dimension of $C$ as a polyhedral set is $n_0 -k$, where $k$ is the number of $0$'s 
in $s_C$. 

Moreover, we state the following additional result (implicit in \citet{Masden}), which tells us that the {\em presence} of a ternary activation pattern is stable under (almost all) sufficiently small perturbations of the parameter:\footnote{Note that the {\em absence} of a ternary activation pattern is not necessarily stable in this way. See the definition of combinatorial stability and related discussion in \cite{GLMW}.} 

\begin{proposition}
\label{prop:ternactivestable} Let $\theta' \in \Omega$ be a generic, supertransversal parameter, $C' \in \mathcal{C}(F_{\theta'})$ a cell in the corresponding polyhedral complex and $s_{C'}$ its associated ternary activation pattern. 
There exists an open neighborhood $N$ of $\theta' \in \Omega$ 
such that for each $\theta \in N$,  $\theta$ is generic, supertransversal, and there exists a non-empty cell $C$ in $\mathcal{C}(F_{\theta})$ with ternary activation pattern $s_{C} = s_{C'}$. 

\end{proposition}

\begin{proof} 


The assumption that $\theta'$ is generic and supertransversal tells us that the ternary activation pattern of a cell $C'$ in $\mathcal{C}(F_{\theta'})$ gives us a precise recipe  for realizing $C'$ as the intersection of bent hyperplanes and ``bent" half-spaces.\footnote{The complement of a bent hyperplane is the union of at most two open connected components. These are what we mean by bent half-spaces. Note that a bent hyperplane may be empty, in which case exactly one of the two bent half-spaces is also empty.}  Moreover, if $C'$ has dimension $(n_0-k)$, $s_{C'}$ will have $k$ $0$'s (corresponding to $k$ intersecting bent hyperplanes) and  $(n_0-k)$ $\pm 1$'s (corresponding to $(n_0 - k)$ \color{black} intersecting bent half-spaces). We also note (cf. Lem. 12 of \citet{Masden} and Lemma \ref{lem:genericfullmeasure} that the set of generic, supertransversal parameters is full measure in parameter space.

Now suppose that $C'$ is a non-empty cell in $\mathcal{C}(F_{\theta'})$, and $p'$ is a point in $\mbox{int}(C')$. 
Let $\mathcal{H}_0$ (resp., $\mathcal{H}_\pm$) denote the set of bent hyperplanes of $F_{\theta'}$ associated to ternary label $0$ (resp., ternary labels $\pm 1$) in $s_{C'}$.

By transversality on cells, there is an open neighborhood $N_0$ of the subset of the parameters  defining the bent hyperplanes in $\mathcal{H}_0$ for which every parameter $\theta \in N_0$ defines a collection of $|\mathcal{H}_0|$ bent hyperplanes with intersection that is both transverse on cells and non-empty.

Moreover, since $p'$ has positive minimal distance to every bent hyperplane in $\mathcal{H}_{\pm}$, and every such bent hyperplane is closed (though not necessarily compact), there is some positive $\delta$ for which a neighborhood of $p'$ of radius $\delta$ contains only the bent hyperplanes  in $\mathcal{H}_0$. This implies that there is a sufficiently small open neighborhood $N_{\pm}$ of the parameters defining the bent hyperplanes in $\mathcal{H}_\pm$ for which $C'$ is in the same bent half-space for the bent hyperplanes in $\mathcal{H}_\pm$ for every parameter in $N_{\pm}$.

Letting $N=N_0 \cap N_\pm$ and further restricting to a neighborhood with generic parameters (Lemma \ref{lem:genericfullmeasure}) if necessary, \color{black} we obtain a neighborhood of $\theta' \in \Omega$ for which every $\theta \in N$ is generic, supertransversal, and contains a non-empty cell $C$ with $s_C = s_C'$, as desired.

\end{proof}

Recall (Definition \ref{def:TPIC} that a generic, supertransversal parameter $\theta$ satisfies TPIC if every pair of bent hyperplanes in adjacent layers has non-empty transverse intersection.

\begin{corollary} \label{p:TPICOpenCondition}
For any architecture, TPIC is an open condition. That is, if $\theta \in \Omega$ is a generic, supertransversal parameter satisfying TPIC, then there exists an open neighborhood of $\theta \in \Omega$ on which all parameters satisfy TPIC.
\end{corollary}

The proof of Lemma \ref{lem:openpaths} showed that for a parameter $\theta$ and point $x$ in the intersection of two bent hyperplanes, the polynomials for the four associated sign sequences are distinct.  However, it did not show that there is a neighborhood $N \subset \Omega$ of $\theta$ such that all four sign sequences are actually realized on $\mathbb{R}^{n_0}$ near $x$ for all $\theta' \in N$. The next Lemma combines the persistence of cells realizing sign sequences given by Proposition \ref{prop:ternactivestable} with Lemma \ref{lem:openpaths}.

\color{black}

\begin{lemma}
\label{lem:LRAinanhd}
Suppose $\theta' \in \Omega$ is a generic, supertransversal parameter, let $x = p_{ij}^{\ell}$ be a point in $\hat{H}_i^{\ell-1} \cap \hat{H}_j^{\ell}$, and let $T_{ij}^\ell$ be the union of the closures of the four $\pm$-activation regions adjacent to $x=p_{ij}^{\ell}$ (as in Lemma \ref{lem:openpaths}). If $\Gamma_x\setminus \Gamma_{x,*}$ is non-empty, then there is an open neighborhood $N$ of $\theta' \in \Omega$ and an algebraic, measure $0$ set $S \subset N$ for which every $\theta \in N \setminus S$ satisfies:
\begin{enumerate}[label=(\roman*)]
    \item There are four non-empty $\pm$-activation regions of $F_\theta$ with the same ternary labelings as those in $T_{ij}^\ell$
    \item Letting $T_{ij}^\ell(\theta)$ denote the (non-empty) closure of these four $\pm$ activation regions, LRA is satisfied on $T_{ij}^\ell(\theta)$.
\end{enumerate}
\end{lemma}

\begin{proof} Proposition \ref{prop:ternactivestable} tells us that there is an open neighborhood of $\theta'$ for which each $\theta$ in the neighborhood satisfies (i). 
Moreover Lemma 14 of \citet{Masden} and Lemmas \ref{lem:genericfullmeasure} and \ref{lem:openpaths} tell us that  away from a closed (algebraic) set of measure $0$, the parameters $\theta $ in this open neighborhood satisfy LRA on $T_{ij}^\ell(\theta)$, as desired. 
\end{proof}

\begin{lemma}[\citet{RolnickKording}, Thm. 2] \label{lem:PreciseLRA}
Let $\theta \in \Omega$ be a generic, supertransversal parameter satisfying TPIC. For each $i,j,\ell$, choose a point $p_{ij}^{\ell} \in \hat{H}^{(\ell-1)}_i \cap \hat{H}^{\ell}_j$. Let $T_{i,j}^{\ell}$ denote the union of the closures of the four $\pm$--activation regions adjacent to $p_{ij}^{\ell}$, and let $T = \bigcup_{i,j,\ell} T_{i,j}^{\ell}$. If $F_\theta$ satisfies LRA on $T$, then $\theta$ can be recovered from $F_\theta$ up to permutation and positive-rescaling.  
\end{lemma}

\begin{proof}
The proof of Theorem 2 and associated algorithm in \citet{RolnickKording} require only that the LRA is satisfied in the (closures of) the four adjacent $\pm$-activation regions for one point in each relevant transverse pairwise intersection.
\end{proof}

\begin{remark}
    Lemma \ref{lem:PreciseLRA} tells us that in order for a parameter satisfying TPIC to admit no hidden symmetries, it need not satisfy LRA everywhere but only on the union of the closures of all activation regions near the intersection points (the set $T$ in the statement of Lemma \ref{lem:PreciseLRA}. Accordingly, we will say that a parameter $\theta \in \Omega$ satisfying the assumptions of Lemma \ref{lem:PreciseLRA} {\em satisfies TPIC and LRA on a neighborhood of the pairwise intersections}.
\end{remark}

\section{Main Theorem and Proof}

\begin{theorem} \label{thm:TPIC-formal} Let $(n_0, \ldots, n_d)$ be a neural network architecture satisfying $(n_0 = k) \leq n_\ell$ for all $\ell$, and let $\Omega = \mathbb{R}^D$ denote its parameter space. There exists a positive measure subset $Y \subset \Omega$ for which each $\theta \in Y$ satisfies TPIC and LRA on a neighborhood of the pairwise intersections, hence has no hidden symmetries.
\end{theorem}
\color{black}

\begin{proof}[Proof of Theorem \ref{thm:TPIC-formal}] In the course of the proof we will need the following additional notation. Let $\mathcal{A}^\ell = \{H_1^{\ell}, \ldots, H_{n_{\ell}}\}$ be a hyperplane arrangement in $\mathbb{R}^{n_{\ell-1}}$. If the initial coordinate $k$--plane $\mathbb{R}^{n_\ell}_k$ is transverse to $H_i^{\ell}$, then the intersection, $H_i^{\ell} \cap \mathbb{R}^{n_{\ell-1}}_k$, is a hyperplane in $\mathbb{R}^{n_{\ell-1}}_k$. In this case we will use:
\begin{equation}\label{eqn:restricttoRk}
K_i^{\ell}:= H_i^{\ell} \pitchfork \mathbb{R}^{n_{\ell-1}}_k
\end{equation} 
to denote this hyperplane in $\mathbb{R}^{n_{\ell-1}}_k$. If $\mathcal{A}^{\ell} \pitchfork_c \mathbb{R}^{n_{\ell-1}}_k$, then this implies that all of the hyperplanes in $\mathcal{A}^{\ell}$ intersect $\mathbb{R}^{n_{\ell-1}}_k$ transversely, and we will use $\mathcal{K}^{\ell}$ to denote the corresponding hyperplane arrangement in $\mathbb{R}^{n_{\ell-1}}_k$.

We now proceed to prove the theorem  by construction. Our strategy will be to find a particular generic, supertransversal parameter $\theta' \in \Omega$ satisfying TPIC and LRA on a neighborhood of the intersections. It will then follow by Lemmas \ref{lem:LRAinanhd} and \ref{lem:PreciseLRA} that every parameter $\theta$ away from a measure zero set in a neighborhood of $\theta'$ satisfies TPIC and LRA on a neighborhood of the intersections. The conclusion of the theorem will then follow.

 Our construction will be \color{black} by induction on $d$. We will find it convenient to prove the following strictly stronger (and unavoidably technical) conclusion, since it helps with the inductive construction:  There exists a generic, supertransversal choice of parameters whose associated sequence of polyhedral refinements $\mathcal{C}(F_1 = F_{(1)}) \succeq \ldots \succeq \mathcal{C}(F=F_{(d)})$ contains a nested sequence of unbounded $k$--cells $\mathcal{S}_1 \supseteq \ldots \supseteq \mathcal{S}_d$ for which the $\ell$-th ternary labeling on $\mathcal{S}_\ell$ is $s^{\ell} = (\underbrace{+1, \ldots, +1}_{k}, \underbrace{-1, \ldots -1}_{n_\ell - k})$ for all $\ell \leq d$ and which satisfies the additional conditions that 
\begin{enumerate}[label=(\roman*)]
	\item $\mathcal{S}_\ell \pitchfork_c (\hat{H}_i^{\ell + 1} \pitchfork_c \hat{H}_j^{\ell + 2}) \neq \emptyset$ for all $i,j$ when $\ell \leq d-2$, and
	\item $\mathbb{R}^{n_{\ell}}_k \pitchfork_c \mathcal{C}(\mathcal{A}^{\ell+1})$, and the preimage of the bounded subcomplex of $\mathcal{C}(\mathcal{K}^{\ell+1})$ under the map $F_{(\ell)}$ is contained in the interior of $\mathcal{S}_{\ell}$ for all $\ell \leq d-1$. 
\end{enumerate}

Note that condition (i) above implies TPIC,  and the assumption that the $\ell$th ternary labeling on $\mathcal{S}_\ell$ has $k$ $+1$'s for all $\ell$ is enough to guarantee that for each pairwise intersection $x = p_{ij}^\ell$ the set $\Gamma_x \setminus \Gamma_{x,*}$ referenced in Lemma \ref{lem:openpaths} is non-empty, hence LRA will be satisfied in a neighborhood of the intersections. 

\color{black}

When $d=1$, choose $\mathcal{A}^1 = \{H^1_1, \ldots, H^1_{n_1}\}$ to be any generic arrangement of hyperplanes. Choose any unbounded $k$--cell of 
 $\mathcal{C}\left(\mathcal{A}^{1}\right)$ and call it $\mathcal{S}_1$. Note that we can alter the ordering and co-orientations on the hyperplanes of $\mathcal{A}^{1}$ (without affecting the arrangement) to ensure that the first $k$ neurons of $F^{1}$ are active on $\mathcal{S}_1$ and the remainder are inactive. That is, the ternary labeling on $\mathcal{S}_1$ is $s^{1} = (\underbrace{+1, \ldots, +1}_{k}, \underbrace{-1, \ldots -1}_{n_1 - k})$. The rest of the conditions are vacuously true.

Now suppose $d=2$. By Corollary \ref{cor:AffIsomorph}, we know that $F^{1}$ restricted to the interior of $\mathcal{S}_1$ is a homeomorphism onto the interior of the $k$--dimensional polyhedral set $F^{1}(\mathcal{S}_1)$ in $\mathbb{O}^{\geq 0}_k \subseteq \mathbb{R}^{n_1}$. Moreover, $F^{1}(\mathcal{S}_1)$ is unbounded by Proposition \ref{prop:unbddimage}. Lemma \ref{lem:intunbddcell} then tells us that there exists a positive-axis sufficiently high-bias hyperplane $H \subset \mathbb{R}^{n_1}$ for which $F^{1}(\mathcal{S}_1) \pitchfork_c H \neq \emptyset$. 

Proposition \ref{prop:perturb} ensures we can choose sufficiently small perturbations $H_1^{2}, \ldots, H_{n_2}^{2}$ of  $H$ so that all of the hyperplanes in $\mathcal{A}^{2}$ are transverse to $\mathcal{S}_1 \subseteq \mathbb{R}^{n_1}$, the parameters associated to $\mathcal{A}^{1}, \mathcal{A}^{2}$ are generic and supertransversal, and the bounded subcomplex of the restricted hyperplane arrangement $\mathcal{K}^{2} = \mathcal{A}^2 \cap \mathbb{R}^{n_1}_k$ is contained in $\mbox{int}(F^{1}(\mathcal{S}_1))$.  Another application of Corollary \ref{cor:AffIsomorph} allows us to conclude that the preimage of the bounded subcomplex of $\mathcal{C}(\mathcal{K}^2)$ is contained in the interior of $\mathcal{S}_1$, and hence the technical inductive conclusion (ii) is satisfied. Since $d = 2$, the technical inductive conclusion (i) is vacuous, so the $d=2$ case is proven.

Now assume $d > 2$. By the inductive assumption, there exists a generic, supertransversal choice of parameters for $F^{1}, \ldots, F^{d-1}$ for which the technical inductive conditions above are satisfied for all $\ell$ up through $d-1$. Therefore, we need only choose generic parameters for $F^{d}$ that are supertransversal to all previous choices of parameters and for which in the polyhedral refinement $\mathcal{C}(F_{(d-1)})$ of $\mathcal{C}(F_{(d-2)})$ we have identified an unbounded $k$--cell $\mathcal{S}_{d-1} \subseteq \mathcal{S}_{d-2}$ with ternary labeling $s^{d-1} = (\underbrace{+, \ldots, +}_{k}, \underbrace{-, \ldots -}_{n_{d-1} - k})$, 
\begin{enumerate}
	\item $\mathcal{S}_{d-2} \pitchfork_c (\hat{H}_i^{d-1} \pitchfork_c \hat{H}_j^{d}) \neq \emptyset$ for all $i,j$, and 
	\item $\mathbb{R}^{n_{d-1}}_k \pitchfork_c \mathcal{C}(\mathcal{A}^{d})$, and the preimage of the bounded subcomplex of $\mathcal{C}(\mathcal{K}^{d})$ under the map $F_{(d-1)}$ is contained in the interior of $\mathcal{S}_{d-1}$. 
\end{enumerate}	
	
Proceed by choosing any unbounded $k$--cell contained in $\mathcal{S}_{d-2}$ in the polyhedral refinement $\mathcal{C}(F_{(d-1)})$ and call it $\mathcal{S}_{d-1}$. As before, we are free to choose the ordering and co-orientations on the hyperplanes of $\mathcal{A}^{d-1}$ without altering $\hat{\mathcal{A}}^{d-1}$ or the polyhedral refinement $\mathcal{C}(F_{(d-1)})$. So we can arrange for $\mathcal{S}_{d-1}$ to have the desired ternary labeling. 

By Corollary \ref{cor:AffIsomorph} we know that $F_{(d-1)}$ restricted to the interior of $\mathcal{S}_{d-1}$ is a homeomorphism onto the interior of an unbounded $k$--dimensional polyhedral set in $\mathbb{O}^{\geq 0}_k \subseteq \mathbb{R}^{n_{d-1}}$, and we can therefore choose a sufficiently high bias positive-axis hyperplane $H$ in $\mathbb{R}^{n_{d-1}}$ so that $F_{(d-1)}(\mathcal{S}_{d-1}) \pitchfork_c H \neq \emptyset$. By Proposition \ref{prop:perturb} we can find small perturbations $H_1^{d}, \ldots, H_{n_d}^{d}$ such that the preimage of the bounded subcomplex of the restricted generic hyperplane arrangement is contained in $F_{(d-1)}(\mathcal{S}_{d-1})$, which we may assume is unbounded by Proposition  \color{red} \ref{prop:unbddimage}. \color{black}

Moreover, Lemma \ref{lem:intunbddcell} tells us that since $H \subseteq \mathbb{R}^{n_{d-1}}$ is a sufficiently high bias positive-axis hyperplane, \[F^{d-1}(H^{d-1}_i) \pitchfork H \neq \emptyset\] for all $i$. Since non-empty transverse intersection is an open condition, we also know that \[F^{d-1}(H^{d-1}_i) \pitchfork_c H^{d}_j \neq \emptyset\] for all $i,j$. An application of Lemma \ref{lem:transpullback} then tells us that \[H^{d-1}_i \pitchfork \grave{H}^{d}_j \neq \emptyset\] for all $i,j$.

We would now like to conclude that condition (i) is satisfied for $\ell = d-2$, but although we know that $F_{(d-1)}(\mathcal{S}_{d-2}) \pitchfork_c H_i^{d-1} \neq \emptyset$ and $H^{d-1}_i \pitchfork_c \grave{H}^d_j \neq \emptyset$ we don't yet know that $H^{d-1}_i \cap \grave{H}^d_j$ are contained in the unbounded polyhedral set $F_{(d-2)}(\mathcal{S}_{d-2})$, so we cannot yet conclude that the three-fold intersections $\mathcal{S}_{d-2} \pitchfork_c (\hat{H}_i^{d-1} \pitchfork_c \hat{H}_j^{d})$ are non-empty. 

To arrange for this, choose one point of intersection from $H^{d-1}_i \cap \grave{H}^d_j$ for each $i,j$. This is a finite, hence bounded, set. Call it $B$.

We now appeal to Lemma \ref{lem:TransByCvxCone}, which tells us that there is some vector $v \in \mbox{Cone}(F_{(d-2)}(\mathcal{S}_{d-2}))$ that can be added to $B$ so that $B + v \in \mbox{Cone}(F_{(d-2)}(\mathcal{S}_{d-2}))$. Since we can achieve this translation by altering just the biases of $F^{d}$ and $F^{d-1}$, we have now arranged that condition (i) is satisfied. The inductive proof that our construction satisfies TPIC is now complete,  and hence a positive measure subset of $\Omega$ has no hidden symmetries, as desired. \color{black}

\end{proof}

\color{black}

\color{black}


\section{Functional dimension, Fibers, and Symmetries}

Roughly speaking, the functional dimension of a parameter $\theta \in \Omega_{n_0,\ldots,n_d}$ is the dimension of the space of functions $F_{\theta}$ realizable by infinitesimally perturbing $\theta$. We will now make this more precise.  
Suppose $Z = \{z_1,\ldots,z_k\}$ is a finite collection of points in the domain $\mathbb{R}^{n_0}$ of $F_{\theta}$. For any $\theta$, we may record the result of evaluating the map $F_{\theta}$ at all $k$ points of $Z$ as a single ``unrolled'' vector, i.e. 
$$E_Z(\theta) \coloneqq (F_{\theta}(z_1),\ldots,F_{\theta}(z_k)) \in \mathbb{R}^{kn_d}.$$
We can then measure how many degrees of freedom we have to vary this data locally at $\theta$ by considering the rank of the total (Jacobian) derivative of the map $\theta \mapsto E_Z(\theta)$, $\textrm{rank}(\boldsymbol{J}E_z(\theta))$.

Of course, because ReLU is not differentiable at $0$, there is a (Lebesgue) measure $0$ set of pairs $(\theta,x)$ in $\Omega_{n_0,\ldots,n_d} \times \mathbb{R}^{n_0}$ at which the total derivative does not exist.  Consequently, we will wish to restrict our attention to pairs $(\theta,x)$ at which all relevant partial derivatives exist. 

\begin{definition}[\citet{GLMW}] \label{def:parametricallysmooth}
  A point $x \in \mathbb{R}^{n_0}$ is \emph{parametrically smooth} for a parameter $\theta \in \Omega_{n_0,\ldots,n_d}$ if $(\theta,x)$ is a smooth point for the \emph{parameterized family}
  $\mathcal{F}: \Omega_{n_0,\ldots,n_d} \times \mathbb{R}^{n_0} \to \mathbb{R}^{n_d}$ defined by 
  $\mathcal{F}(\theta,x) = F_{\theta}(x)$, where $F_{\theta}:\mathbb{R}^{n_0} \to \mathbb{R}^{n_d}$ is the neural network map determined by the parameter $\theta$.
\end{definition}


\begin{definition} \label{def:functionaldimension}
The functional dimension of a parameter $\theta \in \Omega_{n_0,\ldots,n_d}$ is
$$\textrm{dim}_{\textrm{fun}}(\theta) \coloneqq \sup \left \{ \textrm{rank}(\boldsymbol{J}E_Z(\theta)) \mid Z\subset \mathbb{R}^{n_0} \textrm{ is a finite set of parametrically smooth points for } \theta \right \}$$
where the supremum is taken over all sets $Z$ consisting of finitely many points in $\mathbb{R}^{n_0}$.
\end{definition}

In practice, in experiments such as those in \S\ref{sec:experiments}, we ignore the issue of differentiability, assuming that all points in a randomly selected set $Z \subset \mathbb{R}^{n_0}$ are parameterically smooth for the parameter; the assumption is supported by the fact that $\mathcal{F}$ is smooth except on a set of $0$ measure, 

\begin{definition}
The \emph{fiber}  (with respect to the realization map) of a parameter $\theta \in \Omega$ is the set $\{\tilde{\theta} \in \Omega \mid F_{\tilde{\theta}} = F_{\theta}\}.$ 

\end{definition}

Recall that elements of {\bf (P)}, the set  of permutations, and elements of {\bf (S)}, the scalings, act on $\Omega$. Moreover, these actions commute, so the semigroup generated by {\bf (P)} and {\bf (P)} equals $\{p \circ s \mid p \in \textbf{(P)}, s \in \textbf{(S)}\}$. The following definition formalizes the notion that a \emph{hidden symmetry} describes two parameters that define the same function, but do not differ by elements of {\bf (P)} and {\bf (S)}. 

\begin{definition} \label{def:noHiddenSym}
A parameter $\theta \in \Omega$ \emph{has no hidden symmetries} if the semigroup generated by {\bf (P)} and {\bf (S)} acts transitively on the fiber of $\theta$.  
\end{definition}

In other words, $\theta$ has no hidden symmetries if, given any two parameters $\theta_1$ and $\theta_2$ such that $F_{\theta_1} =F_{\theta_2} = F_{\theta}$, there exists $p \in \textbf{(P)}$ and $s \in \textbf{(S)}$ such that $p \circ s(\theta_1) = \theta_2$.
\color{black}

\begin{definition}  \label{def:localsymmetry}  \ 
\begin{enumerate}
\item A \emph{pointwise symmetry} of $\theta$ is a permutation (i.e a bijection, not necessarily continuous) of the fiber of $\theta$.
\item A \emph{local symmetry} of $\theta$ is a homeomorphism of the fiber of $\theta$ equipped with the subspace topology (as a subset of $\Omega$).  
\item A \emph{global symmetry} of $\Omega$ is a homeomorphism $T:\Omega \to \Omega$ such that $F_{T(\theta)} = F_{\theta}$ for all $\theta \in \Omega$. 
\end{enumerate}
\end{definition}

As shown in \citet{GLMW}, there exists a fiber and a permutation of that fiber that cannot be extended to a homeomorphism of an open neighborhood of that fiber.  This is because it is possible for a fiber to contain two parameters $\theta_1$ and $\theta_2$ (of the same architecture) such that  any arbitrarily small neighborhood of $\theta_1$ gives rise to functions that cannot be realized by parameters in arbitrarily small neighborhoods of $\theta_2$.  Said otherwise, there exist pointwise symmetries that cannot be extended to local symmetries. Our definition of \emph{no hidden symmetries} (Definition \ref{def:noHiddenSym}) is a statement about pointwise symmetries -- that every pointwise   symmetry of $\theta$ is the restriction to the fiber of $\theta$ of an element $p \circ s$ of the group of global symmetries generated by \textbf{(P)} and \textbf{(S)}.

The theoretical upper bound on functional dimension (\cite{GLMW}) comes from taking account of (only) the well-known global symmetries \textbf{(P)} and \textbf{(S)}. The intuition is that the set of functions realizable by parameters in a small neighborhood $U$ of $\theta$ should be modeled by the quotient of $U$ by the equivalence relation defined by $\bf{(P)}$ and $\bf{(S)}$.

\begin{lemma} \label{lem:AchievesUpperBound}
Let $U$ be an open ball in $\Omega$ such that if any two points $u_1, u_2 \in U$ satisfy $F_{u_1} = F_{u_2}$ if and only if $u_2 = p \circ s (u_1)$ for some $p \in \textbf{(P)}$ and $s \in \textbf{(S)}$. Then the functional dimension of any parameter $\theta \in U$ attains the theoretical upper bound. 
\end{lemma}

\begin{proof}
Denote by $\sim$ the equivalence relation on $U$ defined by the semigroup generated by \textbf{(P)} and \textbf{(S)}.  Denote by $\mathcal{F}\vert_U$ the set of functions realizable by parameters in $U$, equipped with the compact-open topology. 
The condition that $F_{u_1} = F_{u_2}$ if and only if $u_2 = p \circ s (u_1)$ for some $p \in \textbf{(P)}$ and $s \in \textbf{(S)}$ is equivalent to the statement that the quotient space $U / \sim$ is homeomorphic to $\mathcal{F}_U$.  It follows that for any parameter $\theta \in U$, $\textrm{dim}_{\textrm{fun}}(\theta)$ is the Euclidean dimension of the quotient space $U / \sim$. The result follows. 
\end{proof}

\subsection{Proof of Proposition \ref{prop:containedInSubspace}} \label{ss:mechanismsProofs}

 \color{black}
   The following Lemma will be used to prove Proposition \ref{prop:containedInSubspace}.
   
\begin{lemma} \label{l:rotate}
Let $S$ be a hyperplane in $\mathbb{R}^n$.   Let $A:\mathbb{R}^{n} \to \mathbb{R}^{1}$ be an affine-linear map given by 
$$A(\vec{x}) = \vec{x} \cdot \vec{n}_H - b_H.$$ 
Let $\vec{o} \in \mathbb{R}^n$ be a vector orthogonal to $S$ (i.e. if $\vec{s}_1, \vec{s}_2$ are two points in $S$, then $(\vec{s}_2-\vec{s}_1) \cdot \vec{o} = 0$).  For each $t \in \mathbb{R}$, define the affine-linear map $A_t:\mathbb{R}^n \to \mathbb{R}$ by 
$$A_t(\vec{x}) = \vec{x} \cdot (\vec{n}_H + t \vec{o}) - b_H - t\vec{s}_H \cdot \vec{o}$$
for any fixed point $\vec{s}_H \in S \cap \{x \mid A(x) = 0\}$.
Then $A$ and $A_t$ coincide on $S$. 
\end{lemma}

\begin{proof}
Since $A, A_t$ are affine-linear maps, it suffices to show that they agree at a point (in particular, at the point $\vec{s}_H$) and that $$A(\vec{s}_1) - A(\vec{s}_2) = A_t(\vec{s}_1) - A_t(\vec{s}_2)$$ for all $\vec{s}_1,\vec{s}_2 \in S$.  
First, observe that $$A(\vec{s}_H) =   \vec{s_H} \cdot \vec{n}_H - b_H = 0$$ by definition, and 
$$A_t(\vec{s}_H) = \vec{s_H} \cdot (\vec{n}_H + t \vec{o}) - b_H - t\vec{s}_H \cdot \vec{o} =   (\vec{s_H} \cdot \vec{n}_H - b_H) + s_H \cdot t \vec{o} -  t\vec{s}_H \cdot \vec{o} = 0. $$

Next, for any points $s_1,s_2 \in S$, 
$$A(\vec{s}_1) - A(\vec{s}_2) = (\vec{s}_1 - \vec{s}_2) \cdot \vec{n}_H $$ and 
$$A_t(\vec{s}_1) - A_t(\vec{s}_2) = (\vec{s}_1 - \vec{s}_2) \cdot \vec{n}_H + (\vec{s}_1 - \vec{s}_2) \cdot \vec{o}t =  (\vec{s}_1 - \vec{s}_2) \cdot \vec{n}_H $$ since $\vec{o}$ is perpendicular to $(\vec{s}_1 - \vec{s}_2)$. 
\end{proof}

\begin{proof}[Proof of Proposition \ref{prop:containedInSubspace}]
Fix a nonzero vector $\vec{o}$ that is orthogonal to the hyperplane $S$. 
For $t \in \mathbb{R}$, let $A_t$ be the map constructed in Lemma \ref{l:rotate}, set $f_{H_t} \coloneqq \sigma \circ A_t$, and denote the co-oriented hyperplane associated to $A_t$ by $H_t$. By construction, $f_{H_t}$ and $\eta$ coincide on $S$.   

Denote by $H_t^+$ (resp. $H^+$) the closed nonnegative half-spaces associated to $A_t$ (resp. $A$).  It suffices to show that there exists $\epsilon > 0$ such that $|t| < \epsilon$ implies 
\begin{equation} \label{eq:halfspacesmatch}
H_t^+ \cap \textrm{Im}_{(k)} = H^+ \cap \textrm{Im}_{(k)}.
\end{equation}  
(The desired one-parameter family is then the family parametrized by $t$ with $|t| < \epsilon$).  

For convenience, define $G_k := F^k \circ \ldots \circ F^1$.  Consider the complex $$M \coloneqq G_k(\mathcal{C}(G_k)) \cap H^-.$$
 (Here $\mathcal{C}(G_k)$ denotes the canonical polyhedral complex for $G_k$; we take the image (in $\mathbb{R}^{n_k}$) of this complex under $G_k$, which is itself a polyhedral complex, and then intersect it with the closed half-space $H^-$.)
 By condition \eqref{i:lowdiml}, $$\textrm{Im}_{(k)} \cap \{x \mid \eta(x) = 0\} = \textrm{Im}_{(k)} \cap H \subseteq S.$$  Consequently, each cell of $M$ either contains a cell in $H \cap S$ as a face, or is a positive distance away from $H$.  Consequently, there exists $\epsilon_1 > 0$ such that $|t| < \epsilon_1$ implies the intersection of any bounded cell of $M$ with $H_t^+$ is contained in $S \cap H$ or is empty.  
 By condition \eqref{i:nonparallel}, $|M|$ does not contain any unbounded geometric rays parallel to $H$.  
 \color{black} Consequently, there exists $\epsilon_2$ such that $|t| < \epsilon_2$ implies the intersection of any unbounded bounded cell of $M$ with $H_t^+$ is empty.  Set $\epsilon = \min\{\epsilon_1,\epsilon_2\}$. Then when $|t| < \epsilon$, $H_t^+ \cap \textrm{Im}_{(k)} = H^+ \cap \textrm{Im}_{(k)}$. 
\end{proof}

\section{Supplementary Experiments}
\label{app:experiments}

In this appendix, we consider the effect of varying the number $m$ of sample points when approximating the functional dimension. For a fixed network architecture of depth 4 and width 5, Figure \ref{fig:curve} shows the fraction of networks attaining the maximum possible value for $\fdim(\theta)$, as a fraction of $m$, where $m$ is shown as a multiple of the maximum possible functional dimension. We observe that the curve is very flat in the region of $x=100$ (the value used in Section \ref{sec:experiments}), suggesting that further increasing $m$ would likely not change the results meaningfully.

Figures \ref{fig:app_main} and \ref{fig:app_varying_width} below consider the effect of choosing a much smaller value of $m$. These figures are analogous to Figures \ref{fig:main} and \ref{fig:varying_width}, but with $m$ equal to twice instead of 100 times the maximum possible value for $\fdim(\theta)$. Each figure aggregates results for 20,000 different choices of $\theta \in \Omega$. We note that the bounds obtained on the functional dimension are unsurprisingly somewhat weaker for this very low value of $m$, but the distributions of approximate function dimension still show the same patterns as observed in Section \ref{sec:experiments}, indicating the robustness of our conclusions.

\begin{figure*}[ht]
\begin{center}
\includegraphics[width=0.9\linewidth]{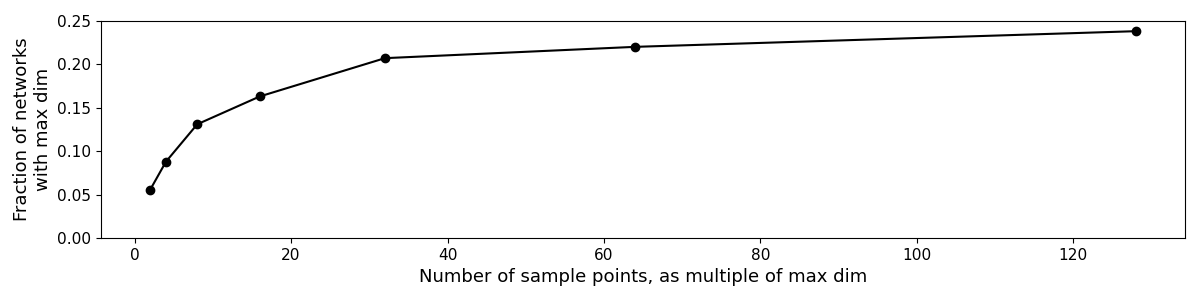}
\caption{This figure shows the fraction of networks attaining the maximum possible value for $\fdim(\theta)$, as a fraction of the number of sample points $m$, where $m$ is shown as a multiple of the maximum possible functional dimension, according to the upper bound given in \citet{GLMW}.}
\label{fig:curve}
\end{center}
\end{figure*}

\begin{figure*}[ht]
\begin{center}
\includegraphics[width=0.75\linewidth]{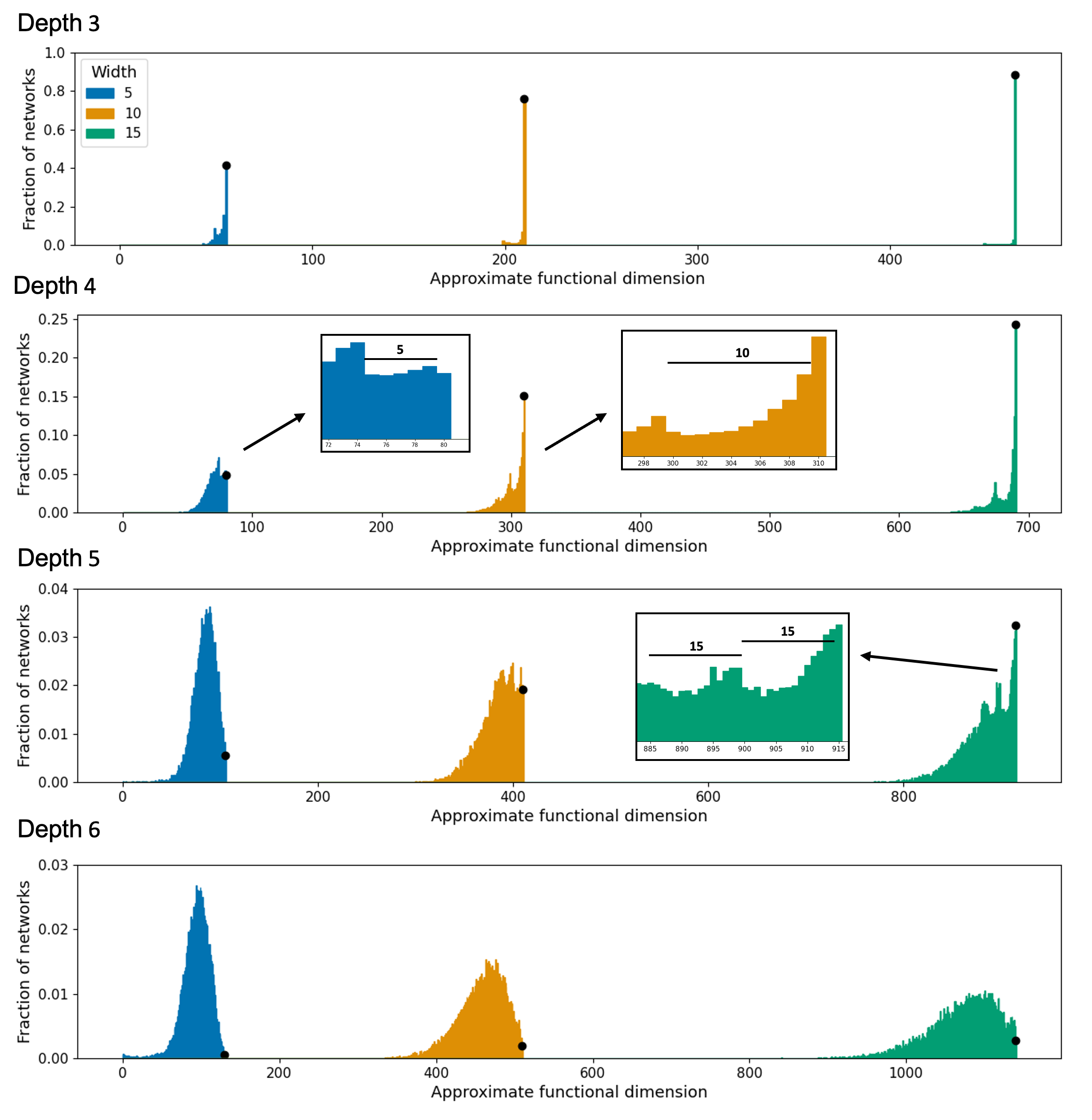}
\caption{For each of various network architectures, we approximate the distribution of functional dimensions as the parameter $\theta\in \Omega$ varies. Different plots show networks of different depths, while different colors in a plot correspond to varying the input dimension and width $n_0=n_1=\cdots = n_{d-1}$. Black dots show the fraction of networks with full functional dimension (no hidden symmetries). We observe as in Figure \ref{fig:main} that the proportion of networks with full functional dimension increases with width and decreases with depth. Insets in the figure zoom in on certain multimodal distributions, showing that modes are again spaced apart by approximately the width of the network.}
\label{fig:app_main}
\end{center}
\end{figure*}

\begin{figure*}[ht]
\begin{center}
\includegraphics[width=0.75\linewidth]{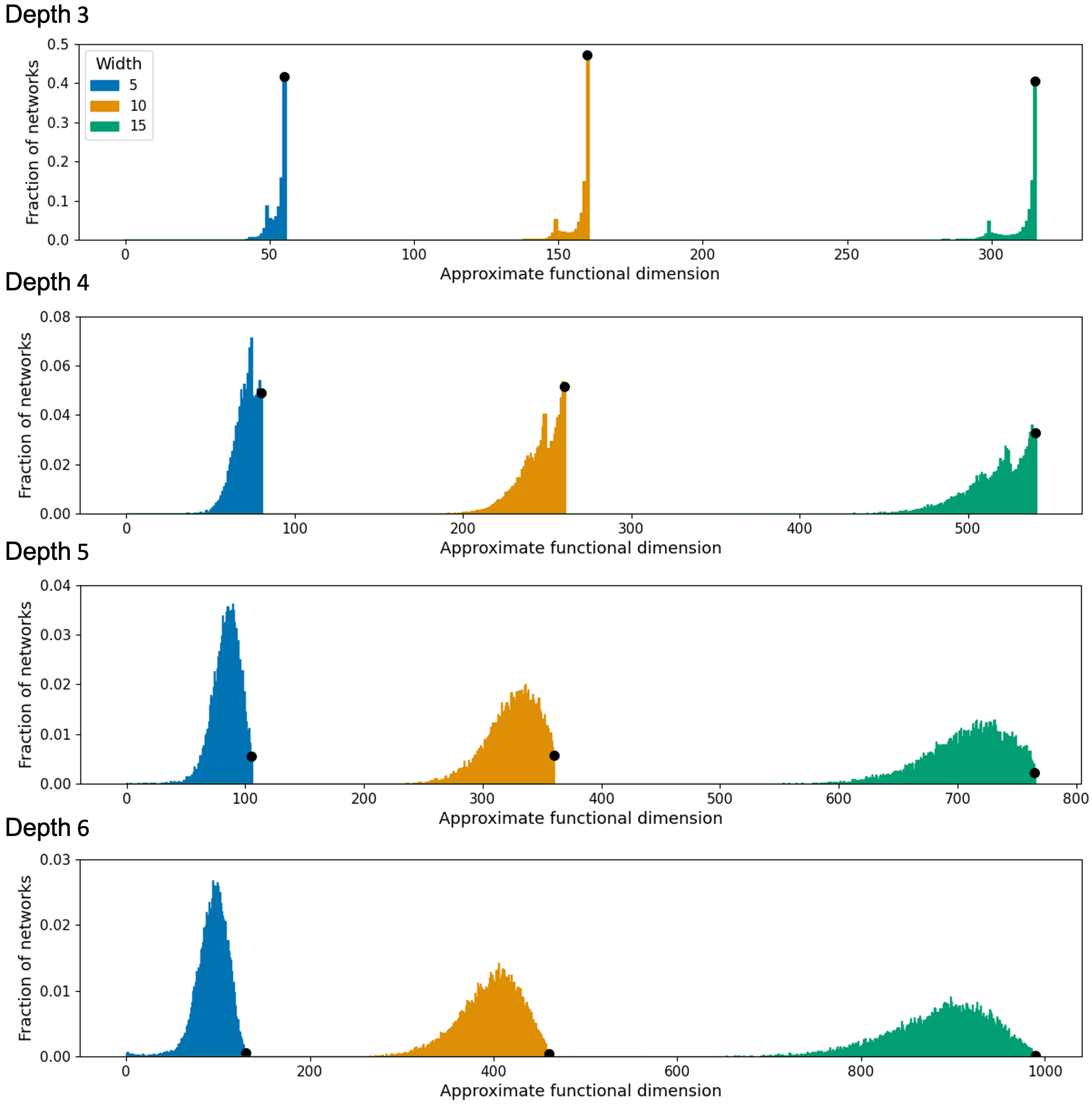}
\caption{This figure shows similar experiments as in Figure \ref{fig:app_main}, but with input dimension fixed at 5 instead of equal to the width of hidden layers. We observe as in Figure \ref{fig:varying_width} that the proportion of networks with full functional dimension decreases with depth, while slightly decreasing with width.}
\label{fig:app_varying_width}
\end{center}
\end{figure*}




\end{document}